\documentclass{article}

\usepackage{microtype}
\usepackage{graphicx}
\usepackage{subfigure}
\usepackage{booktabs} 
\usepackage{multirow}
\usepackage{enumitem} 
\usepackage{hyperref}

\usepackage[accepted]{icml2025}

\usepackage{amsmath}
\usepackage{amssymb}
\usepackage{mathtools}
\usepackage{amsthm}
\usepackage[capitalize,noabbrev]{cleveref}
\usepackage{tcolorbox}

\theoremstyle{plain}
\newtheorem{theorem}{Theorem}[section]
\newtheorem{proposition}[theorem]{Proposition}
\newtheorem{lemma}[theorem]{Lemma}
\newtheorem{corollary}[theorem]{Corollary}
\theoremstyle{definition}
\newtheorem{definition}[theorem]{Definition}
\newtheorem{assumption}[theorem]{Assumption}
\theoremstyle{remark}

\newcommand{\edit}[1]{{\color{black} #1}}

\icmltitlerunning{Navigating the Social Welfare Frontier}

\usepackage[utf8]{inputenc} 
\usepackage[T1]{fontenc}    
\usepackage{hyperref}       
\usepackage{url}            
\usepackage{booktabs}       
\usepackage{amsfonts}       
\usepackage{nicefrac}       
\usepackage{microtype}      
\usepackage{xcolor}         
\usepackage{amsmath}
\usepackage{parskip}
\usepackage{amsthm, amsmath, amssymb, amsfonts}
\usepackage{graphicx}
\usepackage{hyperref}
\usepackage{mathtools}
\usepackage{glossaries}
\usepackage{ulem}
\usepackage{bm}
\usepackage{diagbox}

\newcommand{\R}{\mathbb{R}}

\newcommand{\states}{\mathcal{S}}

\newcommand{\actions}{\mathcal{A}}

\newacronym[longplural=Markov decision processes]{mdp}{MDP}{Markov decision process}
\newacronym[longplural = multi-model MDPs] {mmdp}{MMDP}{multi-model Markov decision process}
\newacronym{cmdp}{CMDP}{contextual Markov decision process}
\newacronym[longplural = partially-observable MDPs]{pomdp}{POMDP}{partially-observable MDP}
\newacronym[longplural = decision makers]{dm}{DM}{decision maker}
\newacronym[longplural = transition probability matrices]{tpm}{TPM}{transition probability matrix}
\newacronym[longplural = discrete time Markov chains]{dtmc}{DTMC}{discrete time Markov chain}
\newacronym{rl}{RL}{reinforcement learning}

\newacronym{rmab}{RMAB}{restless multi-armed bandit}
\newacronym{rlhf}{RLHF}{reinforcement learning with human feedback}

\usepackage{xcolor}

\definecolor{darkgreen}{rgb}{0.0, 0.5, 0.0}

\definecolor{darkgreen}{rgb}{0.0, 0.5, 0.0}

\begin{document}

\twocolumn[
    \icmltitle{Navigating the Social Welfare Frontier: \\ Portfolios for Multi-objective Reinforcement Learning}
    
    
    
    \icmlsetsymbol{equal}{*}
    
    \begin{icmlauthorlist}
    \icmlauthor{Cheol Woo Kim}{equal,har}
    \icmlauthor{Jai Moondra}{equal,geo}
    \icmlauthor{Shresth Verma}{har}
    \icmlauthor{Madeleine Pollack}{mit}
    \icmlauthor{Lingkai Kong}{har}
    \icmlauthor{Milind Tambe}{har,comp}
    \icmlauthor{Swati Gupta}{mit}
    \end{icmlauthorlist}
    
    \icmlaffiliation{har}{Harvard University, Cambridge, USA}
    \icmlaffiliation{comp}{Google Deepmind}
    \icmlaffiliation{geo}{Georgia Institute of Technology, Atlanta, USA}
     \icmlaffiliation{mit}{Massachusetts Institute of Technology, Cambridge, USA}
    
    \icmlcorrespondingauthor{Cheol Woo Kim}{cwkim@seas.harvard.edu}
    
    \icmlkeywords{Machine Learning, ICML}
    
    \vskip 0.3in
    ]



\printAffiliationsAndNotice{\icmlEqualContribution} 


\begin{abstract}
    In many real-world applications of \gls{rl}, deployed policies have varied impacts on different stakeholders, creating challenges in reaching consensus on how to effectively aggregate their preferences. Generalized $p$-means form a widely used class of social welfare functions for this purpose, with broad applications in fair resource allocation, AI alignment, and decision-making. This class includes well-known welfare functions such as Egalitarian, Nash, and Utilitarian welfare. However, selecting the appropriate social welfare function is challenging for decision-makers, as the structure and outcomes of optimal policies can be highly sensitive to the choice of $p$. To address this challenge, we study the concept of an $\alpha$-approximate portfolio in RL, a set of policies that are approximately optimal across the family of generalized $p$-means for all $p \le 1$. We propose algorithms to compute such portfolios and provide theoretical guarantees on the trade-offs among approximation factor, portfolio size, and computational efficiency. Experimental results on synthetic and real-world datasets demonstrate the effectiveness of our approach in summarizing the policy space induced by varying $p$ values, empowering decision-makers to navigate this landscape more effectively.
\end{abstract}

\section{Introduction}\label{sec:intro}

In this paper, we study a reinforcement learning (RL) setting where a deployed policy impacts multiple stakeholders in different ways. Each stakeholder is associated with a unique reward function, and the goal is to train a policy that adequately aggregates their preferences. 

This setting, which is often modeled using multi-objective reinforcement learning (MORL), arises in many RL applications, such as fair resource allocation in healthcare \cite{verma2024}, cloud computing \cite{perez09, hao23} and communication networks \cite{wu18, chen21}. Recently, with the rise of large language models (LLMs), reinforcement learning from human feedback (RLHF) techniques that reflect the preferences of heterogeneous individuals have also been explored \cite{Chakraborty24, zhong2024rlhf, park2024rlhf}. 

Preference aggregation in such scenarios is often achieved by choosing a social welfare function, which takes the utilities of multiple stakeholders as input and outputs a scalar value representing the overall welfare \cite{Yu24, cousins2024welfare, verma2024, fan23, zhong2024rlhf, park2024rlhf, Chakraborty24}.  However, selecting the appropriate social welfare function is a nontrivial task. \edit{Different functions encode distinct fairness criteria, and the resulting policies can lead to vastly different outcomes for stakeholders depending on the choice of the social welfare function.}

In this work, we focus on a class of social welfare functions known as generalized $p$-means, a widely used class of social welfare functions in algorithmic fairness and social choice theory. \edit{Each choice of $p$ represents a distinct notion of fairness,} and the $p$-means unify commonly used welfare functions such as Egalitarian welfare ($p = -\infty$), Nash welfare ($p = 0$) and Utilitarian welfare ($p = 1$), providing a smooth transition between these principles. Notably, this is known to be the only class of social welfare functions that satisfy several key axioms of social welfare, such as monotonicity, symmetry, and independence of scale \cite{roberts_interpersonal_1980, moulin_fair_2003,cousins23, pardeshi2024learning}.

\edit{In practice, the right choice of $p$ is often unclear in advance, and the decision-maker must understand how policies vary with $p$ in order to make informed choices about which $p$ (and thus which policy) to adopt. Small changes in $p$ can sometimes lead to dramatically different policies, and selecting a policy optimized for an arbitrary $p$ can lead to poor outcomes under a different 
$p$ value.} Despite these challenges, much of the existing work assumes a fixed social welfare function — and hence a fixed value of $p$ — is given \cite{hayes_practical_2022}.  

To address this challenge, \edit{we propose a method to compute a small yet representative set of policies that covers the entire spectrum of fairness criteria represented by  $p \le 1$}. Our main algorithm, $p$-\textsc{MeanPortfolio}, sequentially selects finite $p$ values starting from $-\infty$ to 1. These values are chosen so that the optimal policies at these points sufficiently cover the entire range of $p \le 1$ for a given approximation factor $\alpha$. We also propose a computationally efficient heuristic algorithm, which adaptively selects the next $p$ value from the intervals formed by previously chosen $p$ values. 

The portfolios provide a structured summary of approximately optimal policies, allowing decision-makers to navigate the implications of different $p$ values efficiently. \edit{The decision-maker can review the impact of the choice of $p$ on the structure of policies, on relevant dimensions of interest, and make an informed choice about which policy to deploy.}


\begin{figure}[t]
    \centering
    \includegraphics[width=\linewidth,height=3.2in]{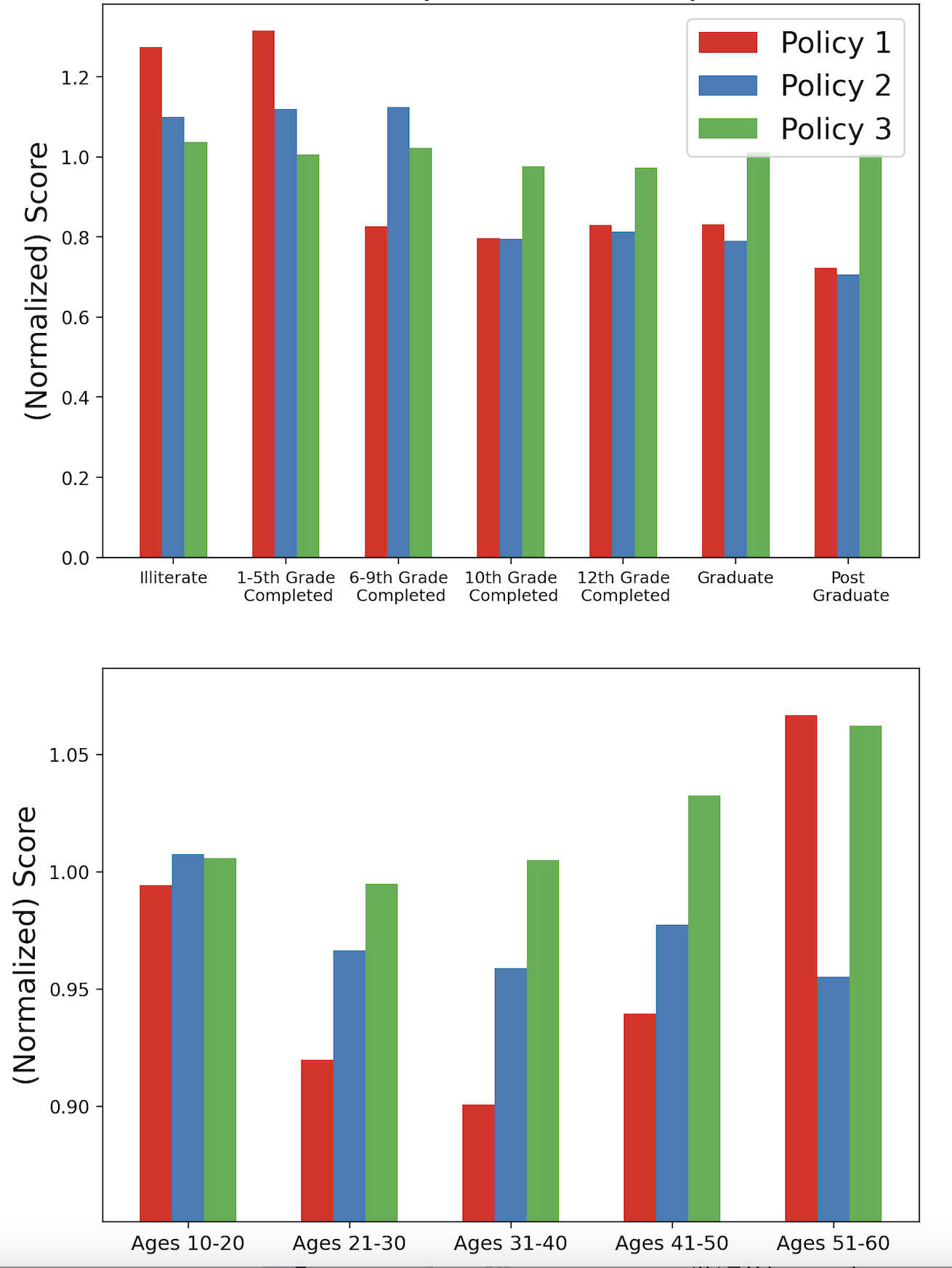}
    \caption{Portfolio of three policies for the healthcare intervention setting (see Section \ref{sec:exp} for details), obtained using Algorithm \ref{alg:portfolios-for-rlhf} with $\alpha = 0.60$. The bar plots show total rewards induced by each policy across different education (top) and age (bottom) brackets. A ``score'' of $1.0$ corresponds to a baseline policy (not shown in the figure) used for comparison. The three policies impact various education and age brackets differently, offering a theoretically sound and diverse set of options for decision-makers.}
    \label{fig: sclm-education-age}
\end{figure}

\subsection{Summary of Contributions}
\label{subsec:contribution}
In this paper, we explore the concept of $\alpha$-approximate portfolios for preference aggregation in MORL. We summarize our contributions as follows:
\begin{enumerate}[leftmargin=*]
    \item We propose Algorithm $p$-\textsc{MeanPortfolio} (Algorithm~\ref{alg:portfolios-for-rlhf}) to compute a finite portfolio of policies that is $\alpha$-approximate for the generalized $p$-means objectives for any value of $p \le 1$. We provide theoretical guarantees, including an upper bound on the portfolio size and the number of MDPs solved to optimality, both expressed in terms of the approximation factor $\alpha$.

    \item We introduce a lightweight heuristic \textsc{Budget-ConstrainedPortfolio} (Algorithm \ref{alg:budget}) that reduces computational costs compared to Algorithm $p$-\textsc{MeanPortfolio} while maintaining high-quality portfolio generation.

    \item We theoretically show that the search for an optimal policy for a given $p$-mean can be efficiently warm-started using the optimal policy for a different $p$-mean.

    \item We evaluate our approach on three different domains, spanning synthetic to real-world problems. Our results show that a small portfolio can achieve near-optimal performance across all $p \le 1$. Moreover, the heuristic \textsc{Budget-ConstrainedPortfolio} constructs portfolios that closely match those produced by Algorithm $p$-\textsc{MeanPortfolio}, while significantly reducing the  computational cost.
\end{enumerate}
The literature in RL with multiple stakeholders has primarily focused on optimizing policies for a given choice of composite objective. However, this work shifts the focus toward constructing a small set of policies with theoretical guarantees. This approach better captures the trade-offs within the actionable policy space, rather than relying on modeling choices on how these objectives should be formulated. 

\subsection{Example}
\label{subsec:example}
An illustrative example we consider in this work is a public health application, where the staff of a healthcare organization aims to train an intervention policy for multiple beneficiaries under a limited budget \citep{verma2024}. This setting can be captured by \glspl{rmab} \cite{whittle98}, which model sequential resource allocation problems under budget constraints. In this context, each beneficiary is associated with a state variable representing their level of engagement with the healthcare program, and the overall state of the problem is the concatenation of the states of all beneficiaries. At each time step, a policy determines which beneficiary to intervene on, subject to the budget constraint.
Depending on the reward function used to train a policy, different socio-demographic groups among the beneficiaries can be prioritized. For example, one reward function might encode the preference to prioritize older populations, while another reward function might encode the preference to prioritize low-income populations. The decision-maker (e.g., healthcare organization staff) must train a policy that balances these conflicting objectives.  

See Figure \ref{fig: sclm-education-age} for a demonstration of a portfolio generated using the proposed method. The portfolio consists of three policies, each affecting stakeholders differently based on age and education levels. This perspective helps staff understand the trade-offs between operational choices. For details on the experiment, refer to Section \ref{sec:exp}.

\section{Related Work}
\label{sec:related}

\textbf{Social Welfare Function and RL.}
In RL with multiple stakeholders, various social welfare functions have been explored, including generalized Gini-Welfare \cite{Yu24, cousins2024welfare}, $p$-means \cite{verma2024, fan23, cousins2024welfare}, proportional-fairness \cite{ju2024achieving}, Nash welfare \cite{mandal2023sociallyfairreinforcementlearning}, among others. \cite{alamdari2024policy} investigated ordinal social welfare functions rather than cardinal ones. 

A related line of work has emerged in RLHF for LLMs, where social choice theory has been used to aggregate multiple reward models. For example, approaches leveraging Nash welfare \cite{zhong2024rlhf}, $\alpha$-fairness \cite{park2024rlhf}, and Egalitarian welfare \cite{Chakraborty24} have been proposed. However, in both RLHF and broader multi-stakeholder RL, existing works typically assume that a fixed social welfare function is given and focus on computing the corresponding optimal policy \cite{hayes_practical_2022}. \edit{As a result, these methods do not offer guidance for settings where the appropriate choice of social welfare function is unclear in advance.}

\textbf{Portfolios in MORL.}
In the broader MORL literature, the concept of portfolios of policies is well-established. A common approach is to compute solutions that approximate the Pareto front \cite{parisi14, moff14, radulescu_multi-objective_2019, hayes_practical_2022}. While the Pareto front provides a characterization of trade-offs between conflicting objectives, it does not assume a specific set of social welfare functions. This generality, while powerful, presents challenges in practical multi-stakeholder settings. That is, the Pareto front does not directly correspond to any specific notion of welfare or utility, making it difficult
to interpret the societal implications of these policies. Furthermore, without an explicit scalarization function to aggregate preferences, decision-makers must choose among Pareto-efficient policies without clear guidance on how to weigh the trade-offs, which is often critical in real-world applications where diverse preferences must be aggregated into a single deployable policy. Additionally, the Pareto front is typically large, making it prohibitively expensive to compute in practice \cite{hayes_practical_2022}. 

A more refined notion of portfolio similar to the one we study is the concept of convex coverage sets. However, this concept is limited to weighted linear combinations of reward functions \cite{roijers13} and does not account for $p-$means. \edit{As a result, existing MORL methods, whether based on the Pareto frontier or weighted linear combinations, do not provide guarantees with respect to $p$-mean welfare functions.  We include additional discussion comparing our method with several well-known MORL algorithms proposed in \cite{yang19, reymond22, Alegre+2023, Yu24} in Appendix \ref{sec:comparisons}.}

\textbf{Portfolios in Optimization.}
In the optimization literature, \cite{drygala2024data} studied portfolios with stochastic guarantees for fixed objectives. Recent work in approximation algorithms also explores the notions of small-sized portfolios that approximately optimize a class of social welfare functions \cite{gupta2024a, gupta2024b, goel2006simultaneous, golovin2008all,chakrabarty2019approximation}, by exploiting the combinatorial properties of facility location, scheduling and set cover problems. While these works operate in well-structured combinatorial settings, our work addresses significantly more complex landscape of RL where the policy space is vast and computing optimal policies is inherently challenging. This necessitates novel algorithmic and theoretical advancements to efficiently construct portfolios with strong guarantees.



\section{Preliminaries}
\label{sec:setting}

In this section, we provide preliminaries on MORL, $p$-mean social welfare functions, and portfolios.

\subsection{Multi-Objective Reinforcement Learning}
\label{subsec:morl}

We consider a multi-objective \gls{mdp}, defined by the tuple $\mathcal{M} = (\states, \actions, {P}, \textbf{R} = (R_i)_{i \in [N]})$, where $\states$ denotes a finite state space, $\actions$ denotes a finite action space, and $P: \states \times \actions \to \Delta_{\states}$ is the transition probability. Additionally, we assume that we start in a fixed initial state $s_1 \in \states$ and $H$ is the time horizon. In the specific MORL context we consider, there are $N$ distinct stakeholders, each associated with a reward function $R_i: \states \times \actions \to \mathbb{R}_{>0}$ for $i \in [N]$. 
A policy $\pi: \states \times [H] \to \Delta(\actions)$ defines a distribution over actions at a given time $h \in [H]$. We use $\tau = (s_1, a_1, \dots, s_H)$ to denote a trajectory, which is the sequence of states and actions from time $1$ to $H$. Throughout, we assume that all reward functions are bounded:
\smallskip
\begin{assumption}[Bounded Reward]\label{assumption:bound}
    There exist strictly positive scalars $U, L$ such that $L \leq R_i(\cdot) \leq U$ for all $i \in [N]$. We call $\kappa := U/L$ the \textit{condition number} of rewards.
\end{assumption}
The $N$ stakeholders could represent different entities depending on the application; for example, different constituent demographics of a democracy, stakeholders in a company, or simply just abstract differences in values or preferences.

\subsection{Social Welfare Function}
\label{subsec:swf}

We consider the following family of social welfare functions $f(\cdot, p): \R^N_{> 0} \to \R_{>0} $ called the \textit{generalized $p$-mean}. It is defined for each $p \in [-\infty, 1]$ and strictly each positive vector $\mathbf{x} = (x_1,\dots,x_N) \in \R^N_{> 0}$ as follows:
\[
f(\mathbf{x}, p) =
\begin{cases}
\min_{i \in [N]} x_i & \text{if} \ p = -\infty, \\
\displaystyle
\left(\frac{1}{N} \sum_{i \in [N]} x_i^p\right)^{\!1/p}
& \text{if } p \not\in \{-\infty, 0\}, \\[1em]
\displaystyle
\left(\prod_{i \in [N]} x_i\right)^{\!1/N}
& \text{if } p = 0.
\end{cases}
\]

\edit{We also provide an illustrative example in Appendix \ref{sec:pmeans} to provide the intuitive impact of various $p$-mean objectives.} We will use the following \textit{monotonicity} property:

\smallskip
\begin{lemma}[Theorem 1, Chapter 3.3.1 \cite{Bullen03}]\label{lem: social-welfare-monotonicity}
    For all strictly positive vectors $\mathbf{x} \in \R^N_{> 0}$, and for all $p, q \le 1$ with $p < q$, we have $f(\mathbf{x}, p) \le f(\mathbf{x}, q)$.
\end{lemma}

The generalized $p$-means provides a method to aggregate heterogeneous preferences by assigning a scalar value to any policy $\pi$. We focus on two aggregation rules, denoted by $\ell \in \{1, 2\}$, that have been previously proposed in the literature. Each combination of an aggregation rule $\ell \in \{1, 2\}$ and a parameter $p$ defines a specific aggregation function, denoted as $v^{(\ell)}(\cdot, p)$. Given a trajectory $\tau = (s_1, a_1, \dots, s_H, a_H)$, let $G_i(\tau) = \sum_{h=1}^{H} R_i(s_h, a_h), i \in [N]$, represent the total reward over the trajectory, and $\textbf{G}(\tau) \in \R^N$ the vector with $G_i(\tau)$ as its $i$th entry. The aggregation functions are defined as follows:
\begin{align}
    v^{(1)}(\pi, p) &= \mathbb{E}_{\tau \sim \pi}\left[f\big(\textbf{G}(\tau), p\big)\right], \label{eqn: first-aggregation-function}\\[10pt]
    v^{(2)}(\pi, p) &= f\big(\mathbb{E}_{\tau \sim \pi}[\textbf{G}(\tau)], p\big). \label{second-aggrehation-function}
\end{align}
In the MORL literature, $v^{(1)}(\cdot, p)$ is referred to as \textit{expected scalarized returns} (ESR), while $v^{(2)}(\cdot, p)$ is known as \textit{scalarized expected returns} (SER). In general, ESR is preferred when the expected total reward from a single execution is the primary focus, whereas  SER is more suitable when policies are executed multiple times and rewards accumulate over iterations. Computing the optimal policy for $v^{(\ell)}(\cdot, p)$ requires specialized algorithms different from standard RL methods, which are not directly applicable. For example, \cite{fan23} and \cite{agarwal22} developed algorithms for $v^{(1)}(\cdot, p)$ and $v^{(2)}(\cdot, p)$, respectively. More detailed comparisons between these scalarization techniques and their solution algorithms are available in \citet{agarwal22}, \citet{roijers13}, \citet{radulescu_multi-objective_2019}, \citet{agarwal22}, and \citet{fan23}.
While these aggregation rules are often studied independently in MORL, our work 
provides a unified framework applicable to both methods. We show that the proposed algorithms apply 
to both settings leveraging similar theoretical foundations.

\subsection{Portfolio of Policies for All $p$-means}\label{subsec:agg_def}


In this section, we introduce the key definitions for portfolios. For brevity, we denote the maximum value function $v^{(\ell)}_*: [-\infty, 1] \to \R$ defined as $\displaystyle v^{(\ell)}_*(p) := \max_{\pi \in \Pi} v^{(\ell)}(\pi, p)$. When $\ell$ and $\Pi$ are clear from context, we denote the optimal policy for $v^{(\ell)}(\cdot, p)$ as $\pi_p$. 

\smallskip
\begin{definition}[$\alpha$-approximation\footnote{The $\alpha$-approximation we refer to is distinct from the $\alpha$-fairness mentioned in Section \ref{sec:related}. To avoid confusion, from this point onward, all instances of $\alpha$ will exclusively refer to the approximation factor.}]\label{def: alpha-approximation}
    Given a $p \le 1$, aggregation rule $\ell \in \{1, 2\}$, and an approximation factor $\alpha \in (0, 1)$, a policy $\pi \in \Pi$ is an $\alpha$-approximation to the aggregation function $v^{(\ell)}(\cdot, p)$ if the value $v^{(\ell)}(\pi, p)$ achieved by $\pi$ is within factor $\alpha$ of the maximum achievable value of $v^{(\ell)}(\cdot, p)$ by policies in $\Pi$, that is,
    \[
        v^{(\ell)}(\pi, p) \geq \alpha \times \max_{\pi' \in \Pi} v^{(\ell)}(\pi', p) = \alpha \times v_*^{(\ell)}(p).
    \]
    When $\ell$ and $\Pi$ are clear from context, we say for brevity that $\pi$ is an $\alpha$-approximate policy for $p$.
\end{definition}

\smallskip
\begin{definition}[Portfolio]
    Given aggregation rule $\ell \in \{1, 2\}$ and approximation factor $\alpha \in (0, 1)$, a set of policies $\Pi' \subseteq \Pi$ is called an $\alpha$-approximate portfolio for aggregation functions $v^{(\ell)}$ if for each $p \le 1$, there exists a policy in $\Pi'$ that is an $\alpha$-approximation to $v^{(\ell)}(\cdot, p)$.
\end{definition}

\section{Portfolio Algorithms}
\label{sec:alg}
In this section, we present our main algorithmic and theoretical results. Before presenting the details of our proposed method, we first define the notion of an oracle, which serves as a key component in our algorithm and analysis. Recall that for a fixed $p$, aggregation function $v^{(\ell)}(\cdot, p), \ell \in \{1, 2\}$, and a given set of policies $\Pi$, the associated welfare maximizing policy can be obtained by solving the following problem:
\begin{equation}
    v^{(\ell)}_*(p) = \max_{\pi \in \Pi} v^{(\ell)}(\pi, p).
    \label{eq:max_aggre}
\end{equation}
The algorithms we develop assume that 
\eqref{eq:max_aggre} can be solved, and we refer to each instance of solving \eqref{eq:max_aggre} as an ``oracle''. This terminology is used to facilitate the discussion of oracle complexity, which quantifies the number of times Problem~\eqref{eq:max_aggre} must be solved to construct a portfolio. 

If $\Pi$ is defined as the set of all possible policies in the MDP, as in the standard RL setting, each oracle call can be computationally expensive. Moreover, achieving exact optimality for Problem~\eqref{eq:max_aggre} may not always be feasible due to limitations in RL algorithms. This means that when measuring the approximation factor of a policy \(\pi\) with respect to an aggregation function \(v^{(\ell)}(\cdot, p)\), the maximum value \(v^{(\ell)}_*(p)\) might not be attainable or computable. In this case, we approximate \(v^{(\ell)}_*(p)\) as the value achieved by applying the given RL algorithm.  This approach measures the approximation factor relative to what is achievable using the algorithm at hand.

In practice, $\Pi$ may be a small set of pre-trained policies, particularly in scenarios where training new policies is infeasible. In such cases, Problem~\eqref{eq:max_aggre} can be solved efficiently by enumerating over the policies $\pi \in \Pi$ and using Monte Carlo simulations to estimate the value $v^{(\ell)}(\pi, p)$.





\subsection{Algorithm $p$-\textsc{MeanPortfolio}}\label{subsec:agg_alg}

\begin{algorithm}[t]
   \caption{\textsc{$p$-MeanPortfolio}$(\mathcal{M}, \Pi, \ell, \alpha)$}
   \label{alg:portfolios-for-rlhf}
\begin{algorithmic}[1]
   \INPUT (i) MDP $\mathcal{M} = (\states, \actions, {P}, \textbf{R} = (R_i)_{i \in [N]})$ with $N$ reward functions, (ii) feasible set of policies $\Pi$ for $\mathcal{M}$, (iii) aggregation rule $\ell \in \{1, 2\}$, and (iv) desired approximation factor $\alpha \in (0, 1)$
   \OUTPUT
      $\alpha$-approximate portfolio $\Pi'$ for the set of 
      aggregation functions $v^{(\ell)}(\cdot, p),  p \le 1$, assuming oracle access to solve problem (\ref{eq:max_aggre})

   \STATE initialize $p_0 = - \frac{\ln N}{\ln (1/\alpha)}$ and $t = 0$
   \STATE initialize portfolio $\Pi' \gets \emptyset$
   \WHILE{$p_t < 1$}
        \STATE add $\pi_{p_t} := {\arg\max}_{\pi \in \Pi} v^{(\ell)}(\pi, p_t)$ to $\Pi'$
        \STATE $p_{t + 1} = \textsc{LineSearch}(v^{(\ell)}, p_t, \Pi, \alpha)$
        \STATE $t \gets t + 1$
   \ENDWHILE
   \STATE \textbf{return} $\Pi'$
\end{algorithmic}
\end{algorithm}

\begin{algorithm}[t]
   \caption{\textsc{LineSearch}$(v^{(\ell)}, p, \Pi, \alpha)$}
   \label{alg: line-search}
\begin{algorithmic}[1]
    \INPUT (i) aggregation function $v^{(\ell)}$ (ii) some $p \in (-\infty, 1)$ (iii) feasible set $\Pi$ of policies, and (iv) desired approximation $\alpha \in (0, 1)$
    \OUTPUT $b^* > p$ such that $\pi_p := {\arg\max}_{\pi' \in \Pi} v^{(\ell)}(\pi', p)$ is an $\alpha$-approximation for 
      $v^{(\ell)}(\cdot, q)$ for all $q \in [p, b^*]$
   \STATE $a \gets p$ and $b \gets 1$
   \STATE initialize $\pi = {\arg\max}_{\pi' \in \Pi} v^{(\ell)}(\pi', p)$
   \WHILE{$v^{(\ell)}(\pi, a) < \alpha \ v_*^{(\ell)}(b)$}\label{step: main-algorithm-while-loop}
       \STATE $q \gets \frac{a + b}{2}$
       \IF{$v^{(\ell)}(\pi, a) \ge \sqrt{\alpha} \ v_*^{(\ell)}(q)$}
           \STATE $a \gets q$ \label{step: line-search-update-lower-bound}
       \ELSE
           \STATE $b \gets q$
       \ENDIF
   \ENDWHILE

   \STATE \textbf{return} $b$

\end{algorithmic}
\end{algorithm}

In this section, we introduce $p$-\textsc{MeanPortfolio} (Algorithm \ref{alg:portfolios-for-rlhf}), which constructs an $\alpha$-approximate portfolio $\Pi'$ for all $p \le 1$. In Theorem \ref{thm: portfolio-with-line-search}, we establish formal bounds on the portfolio size $|\Pi'|$ and the number of oracle calls to solve (\ref{eq:max_aggre}). The full proof is provided in Appendix \ref{sec: omitted-proofs}.

\smallskip
\begin{theorem}\label{thm: portfolio-with-line-search}
    Given an MDP $\mathcal{M}$ with $N$ reward functions with condition number $\kappa$, set $\Pi$ of feasible policies, an aggregation rule $\ell \in \{1, 2\}$, and a desired approximation factor $\alpha \in (0, 1)$, Algorithm $p$-\textsc{MeanPortfolio} returns an $\alpha$-approximate portfolio of policies $\Pi'$ for the set of aggregation functions $\{v^{(\ell)}(\cdot, p): p \le 1\}$. Further,
    \begin{enumerate}
        \item The portfolio size
        \[
            |\Pi'| = O\left(\frac{\ln \kappa}{\ln(1/\alpha)}\right),
        \]
        \item The number of oracle calls by the algorithm is upper bounded by
        \[
            \widetilde{O}\left(\frac{(\ln \kappa)^2 \ln \ln N }{\ln (1/\alpha)}\right),
        \]
        where $\widetilde{O}$ hides all lower order terms.
    \end{enumerate}
\end{theorem}

Here, we explain the high-level ideas behind the algorithm. The algorithm iteratively chooses an increasing sequence of $p$ values $p_0 < p_1 < \ldots < p_K = 1$ using a line search subroutine (Algorithm \ref{alg: line-search}). It ensures that $\pi_{p_t}$ is $\alpha$-approximate for all $p \in [p_t, p_{t + 1}]$, and that $\pi_{p_0}$ is $\alpha$-approximate for all $p \in [-\infty, p_0]$.

\textbf{Line search.} The line search subroutine works as follows: given a $p \le 1$, it seeks to find some $b^* \ge p$ such that $\pi_p$ is an $\alpha$-approximation for all $q \in [p, b^*]$. To achieve this, it maintains lower and upper bounds $a, b$ on $b^*$ with $p \le a < b \le 1$ and iteratively refines these bounds. 

At each iteration, the algorithm checks whether $v^{(\ell)}(\pi_p, a) \geq \alpha v_*^{(\ell)}(b)$ (line 3). If this condition holds, then by the monotonicity property of $p$-means (Lemma \ref{lem: social-welfare-monotonicity}), we must have:
\[
    v^{(\ell)}(\pi_p, b) \ge v^{(\ell)}(\pi_p, a) \ge \alpha \ v^{(\ell)}_*(b).
\]
Then, this holds for any $q \in [a, b]$, implying that $\pi_p$ is $\alpha$-approximate across this interval. Therefore, this procedure can safely output $b^* = b$. This establishes the correctness of the algorithm, as the line search successfully identifies the next $p$ value if it terminates.

Bounds $a, b$ are updated in each iteration as follows (lines 4-8): we query the value $v^{(\ell)}_*\left(q\right)$ for $q = \frac{a + b}{2}$. As before, if $v^{(\ell)}(\pi_p, a)$ exceeds $\alpha$ times $v_*^{(\ell)}(q)$, then we know that $\pi_p$ must be an $\alpha$-approximation for any value in $[a, q]$, and the lower bound $a$ can be tightened as $a \gets q$. Otherwise, the upper bound can be tightened as $b \gets q$.

However, as we show in the formal proof in Appendix \ref{sec: omitted-proofs}, we can in fact ensure faster convergence by slightly modifying this step. Instead of checking whether $v^{(\ell)}(\pi_p, a) \ge  \alpha \times v_*^{(\ell)}(q)$, we check the stronger condition $v^{(\ell)}(\pi_p, a) \ge  \sqrt{\alpha} \times v_*^{(\ell)}(q)$. Since $\sqrt{\alpha} \ge \alpha$, this stricter condition does not affect correctness but accelerates convergence.

\textbf{Oracle complexity.} We briefly sketch how we bound the oracle complexity (i.e., number of oracle calls to solve Problem (\ref{eq:max_aggre})) by bounding the number of oracle calls in each run of \textsc{LineSearch}. As discussed, in each iteration of the line search algorithm, the distance $b - a$ is cut by half since either $b$ or $a$ are updated to $\frac{a + b}{2}$. As we show in Lemma \ref{lem: slope-bound}, this implies that after $j$ iterations of \textsc{LineSearch}, the ratio  $\frac{v^{(\ell)}_*(b)}{v_*^{(\ell)}(a)}$ is upper bounded by $\psi(\kappa, N, \alpha) \times 2^{-j}$, where $\psi$ is some function of condition number $\kappa$, dimension $N$, and approximation factor $\alpha$. By deriving a lower bound on this ratio, we derive an upper bound on the total number of iterations $j$.

\subsection{Heuristic under a Budget Constraint}
\label{subsec:agg_budget}
In $p$-\textsc{MeanPortfolio}, we provide guarantees on the approximation factor $\alpha$. However, achieving these guarantees may require a large number of oracle calls to solve Problem~\eqref{eq:max_aggre} across different values of $p$, which can be impractical when computational resources are severely limited.

One way to reduce oracle calls is to select a smaller $\alpha$. However, while Theorem~\ref{thm: portfolio-with-line-search} provides an upper bound on the number of calls required, the exact number is difficult to determine in advance, making it challenging to balance computational cost with portfolio quality. Alternatively, a budget on the total number of oracle calls can be imposed along $\alpha$, but this introduces its own trade-offs: if $\alpha$ is too large, the algorithm may exhaust calls prematurely, leaving parts of the $p$ range unexplored; if $\alpha$ is too small, it may progress to $p = 1$ too quickly, missing opportunities to refine the portfolio.

To address this limitation, we propose the heuristic \textsc{Budget-ConstrainedPortfolio} (Algorithm \ref{alg:budget}), designed for scenarios with a strict budget $K$ on oracle calls. Unlike $p$-\textsc{MeanPortfolio}, which selects $p$ values in a monotonically increasing order from $-\infty$ to 1, this algorithm dynamically refines the search space based on observed approximation performance. It greedily targets regions where the approximation factor is likely to be weakest, ensuring efficient use of the limited oracle budget.

First, we describe the ideal version of this greedy approach. After $t$ oracle calls, let $\Pi'_t$ denote our current portfolio. The objective at each step is to identify the worst-case $p$ value where the approximation quality of $\Pi'_t$ is the lowest by solving the following problem:
\begin{equation}
\label{eq:greedy}
\mathcal{Q}(\Pi'_t) = {\min}_{p \le 1} \frac{\max_{\pi \in \Pi'_t} v^{(\ell)}(\pi, p)}{v^{(\ell)}_*(p)}.
\end{equation}
However, solving this optimization problem exactly is computationally infeasible, particularly since evaluating the objective for any new $p$ would require an additional oracle call. Instead, \textsc{Budget-ConstrainedPortfolio} approximates this worst-case $p$ using only the information gathered from previous iterations, inspired by the theoretical principles of $p$-\textsc{MeanPortfolio}.

After iteration $t$, the previously selected $p$ values partition the interval $[\infty, 1]$ into at most $t+1$ disjoint intervals:
\[
-\infty < p_{m(1)} < p_{m(2)} < \cdots < p_{m(t)} \leq 1.
\]
Rather than solving Problem~\eqref{eq:greedy} exactly to pinpoint the worst-case $p$, we instead aim to identify the interval where the approximation quality is the weakest. 

First, we rewrite Equation \eqref{eq:greedy} by decomposing the minimization over the intervals:
\begin{equation*}
\mathcal{Q}(\Pi'_t) \approx \min\limits_{l \in [t-1]} 
        \left[
        \min\limits_{p \in [p_{m(l)}, p_{m(l+1)}]} 
        \frac{\max_{\pi \in \Pi'_t} v^{(\ell)}(\pi, p)}
        {v^{(\ell)}_*(p)}
        \right].
\end{equation*} 
The only approximation introduced in this decomposition arises from neglecting the intervals $[-\infty, p_{m(1)}]$ and $[p_{m(t)}, 1]$. To cover the first interval, we can initialize the algorithm with a sufficiently small $p_1$, guaranteeing that the approximation factor remains well-controlled in this region (Theorem \ref{thm: portfolio-with-line-search}). The second interval is covered by explicitly setting $p_2 = 1$.

For each interval $[p_{m(l)}, p_{m(l+1)}]$, we compute its interval approximation factor, denoted as $u(l)$, using the following sequence of approximations:
\begin{equation*}
\begin{aligned}
    &\min\limits_{p \in [p_{m(l)}, p_{m(l+1)}]} 
    \frac{\max_{\pi \in \Pi'_t} v^{(\ell)}(\pi, p)}
    {v^{(\ell)}_*(p)}  \\
    &\approx \min\limits_{p \in [p_{m(l)}, p_{m(l+1)}]} 
    \frac{v^{(\ell)}(\pi_{p_{m(l)}}, p)}
    {v^{(\ell)}_*(p)}  \\
    &\approx \frac{v^{(\ell)}(\pi_{p_{m(l)}}, p_{m(l+1)})}
    {v^{(\ell)}_*(p_{m(l+1)})} := u(l).
\end{aligned}
\end{equation*}
The first approximation follows from the assumption that each interval is well-covered by the policy trained at its left endpoint, similar to the approach used in $p$-\textsc{MeanPortfolio}. The second approximation is justified under the assumption that
$
\frac{v^{(\ell)}(\pi_{p_{m(l)}}, p)}{v^{(\ell)}_*(p)}
$
is a monotonically decreasing function of $p$ in the interval $[p_{m(l)}, p_{m(l+1)}]$. Note that this assumption holds exactly when the interval is sufficiently small, as this function is continuous in $p$ (Lemma \ref{lem: slope-bound}) and attains its maximum value ($=1$) at $p = p_{m(l)}$.

To select the next $p$ value, we identify the interval with the worst (smallest) approximation factor and choose its midpoint:
\begin{equation*}
p_{t+1} = \frac{p_{m(l^*)} + p_{m(l^*+1)}}{2}, \quad l^* = {\arg\min}_{l \in [t-1]} u(l).
\end{equation*}

\begin{algorithm}[tb]
   \caption{\textsc{BudgetConstrainedPortfolio}}
   \label{alg:budget}
\begin{algorithmic}[1]
   \INPUT
      (i) MDP $\mathcal{M} = (\states, \actions,{P}, \textbf{R} = (R_i)_{i \in [N]})$ with $N$ reward functions, (ii) aggregation rule $\ell \in \{1, 2\}$, (iii) feasible set of policies $\Pi$ for $\mathcal{M}$, (iv) budget $K$, and (v) initial $p_0$
   \OUTPUT
      A portfolio $\Pi'$

   \STATE Initialize $\Pi' \gets \emptyset$

   \FOR{$t = 1$ {\bfseries to} $K$}
       \IF{$t = 1$}
           \STATE $p_{t} \gets p_0$
       \ELSIF{$t = 2$}
           \STATE $p_{t} \gets 1$
       \ELSE
           \STATE compute $u(l)$ for $[p_{m(l)}, p_{m(l+1)}], \forall l \in [t-2]$
           \STATE$p_{t} \gets \frac{p_{m(l^*)} + p_{m(l^*+1)}}{2}$, $ l^* = \arg\min_{l \in [t-2]} u(l)$
       \ENDIF
       \STATE Add policy $\pi_{p_t}:= {\arg\max}_{\pi \in \Pi} v^{(\ell)}(\pi, p_t)$ to $\Pi'$
   \ENDFOR
   \STATE \textbf{return} $\Pi'$
\end{algorithmic}
\end{algorithm}

\subsection{Warm-Start for Computational Efficiency}
\label{subsec:warmstart}
Both $p$-\textsc{MeanPortfolio} and \textsc{Budget-ConstrainedPortfolio} involve repeatedly solving Problem \eqref{eq:max_aggre} for different \(p\) values.  To accelerate this process, warm-starting is a natural choice. Specifically, once the optimal policy $\pi_p$ for \(\max_{\pi \in \Pi}v^{(\ell)}(\pi, p)\) is found, it can be used as a warm-starting point for solving \(\max_{\pi \in \Pi}v^{(\ell)}(\pi, q), q > p\), if \(q\) and \(p\) are close. Proposition \ref{prop:warmstart} below offers a theoretical justification for this approach by showing that $v^{(\ell)}(\pi_p, q)$ is close to the optimal value $v^{(\ell)}_*(q)$ whenever $p$ and $q$ are close. The proof is deferred to Appendix \ref{sec: omitted-proofs}.

\begin{proposition}
\label{prop:warmstart}
Let \( q > p \). For \(\ell \in \{1, 2\}\), the following inequality holds:
\[
    \Big| v^{(\ell)}_*(q) - v^{(\ell)}(\pi_p, q) \Big| \le (q-p)UH\kappa \ln{\kappa},
\]
where \(\pi_p\) denotes the  policy \({\arg\max}_{\pi \in \Pi}v^{(\ell)}(\pi, p)\).
\end{proposition}


\section{Numerical Experiments}
\label{sec:exp}

In this section, we present the results of our numerical experiments. Additional details of the experimental setups and the environments are provided in Appendix \ref{sec:expdetail}.
\edit{The code for our experiments can be found at \url{https://github.com/jaimoondra/approximation-portfolios-for-rl/}.}

\subsection{Experimental Setups}
\label{subsec:expsetup}

For a given $\alpha$ and MDP environment, we first compute a portfolio using $p$-\textsc{MeanPortfolio} and compare it against two baseline approaches. Suppose $p$-\textsc{MeanPortfolio} generates a portfolio of size $K$. The first baseline selects $K$ values of $p$ uniformly at random from $[p_0,1]$ and computes their optimal policies, where $p_0$ matches that of $p$-\textsc{MeanPortfolio}. The second baseline randomly samples $K$ policies from $\Pi$. Since both baselines involve randomness, we generate 10 independent baseline portfolios for each setting and report the average performance across these runs. Finally, we run \textsc{BudgetConstrainedPortfolio} with budget $K$.

To evaluate each portfolio $\Pi'$, we compute its approximation factor $\mathcal{Q}(\Pi')$ as defined in Equation \eqref{eq:greedy}. Since computing this value exactly is infeasible, we approximate it using a grid search over $p$, referring to it as the actual approximation. We also compare the number of oracle calls made by $p$-\textsc{MeanPortfolio} and \textsc{BudgetConstrainedPortfolio}. This procedure is repeated for varying $\alpha$ values.

\subsection{Environments}
\label{subsec:expenv}
We conduct experiments across three domains, ranging from synthetic to real-world settings, as described below.

\textbf{Taxi Environment.} We consider a synthetic setting based on the works of \citet{dietterich2000hierarchical,fan2022welfare}. The taxi environment consists of a grid world where a taxi driver serves passengers across $N = 4$ different source-destination pairs. When the agent drops off a passenger at their destination, it earns a reward corresponding to that route. However, since the taxi can only carry one passenger at a time, it must decide which route to prioritize. A fair agent should serve all source-destination pairs equitably, avoiding the neglect of more challenging routes. We train a $p$-mean maximizing policy using Welfare Q-Learning \cite{fan2022welfare}. Finally, we experiment on $v^{(1)}$ and define $\Pi$ as the \edit{set of all possible policies in the MDP}.

\textbf{Resource Allocation after Natural Disaster.} In this synthetic setting, we have a set of $N = 12$ clusters of neighborhoods impacted by a natural disaster. Each cluster is characterized by average household income (high, middle, low), proximity to critical infrastructure (near, far), and population density (high, low), along with distinct post-disaster resource needs. Over a time horizon, a decision-maker must decide how to allocate a limited number of resources to these 12 clusters. The reward function for each cluster is the average of the fraction of unmet need and the fraction of total aid allocated to that cluster. We conduct experiments using $v^{(2)}$, with $\Pi$ defined as a finite set of pre-trained policies.

\textbf{Healthcare Intervention.} We consider a real-world healthcare intervention problem, modeled as an RMAB problem described in Section \ref{subsec:example} \cite{Verma2024b}. ARMMAN \cite{ARMMAN}, an NGO based in India, runs large-scale maternal and child mobile health programs. One of their initiatives delivers critical health information via weekly automated voice messages. To enhance engagement, a limited number of beneficiaries receive direct calls from health workers each week, with call assignments determined by a policy.  \edit{The problem involves  $N = 59$ reward functions,} each prioritizing different socio-demographic groups among the beneficiaries. We conduct experiments using $v^{(2)}$, where $\Pi$ consists of a finite set of pre-trained policies. All experiments are strictly secondary analyses and adhere to ethics board approvals; for further discussion, refer to our impact statement.


\subsection{Results}
\label{subsec:expres}

Table \ref{tab: approximations} presents comparisons of the approximation quality of $p$-\textsc{MeanPortfolio} with random $p$ sampling and random policy sampling baselines, while Table \ref{tab: oracle-call-comparison} reports the number of oracle calls made by $p$-\textsc{MeanPortfolio} and \textsc{BudgetConstrainedPortfolio}. Note that  by design, the number of calls for \textsc{BudgetConstrainedPortfolio} is always the same as the portfolio size.

\textbf{Size}. The portfolios computed using Algorithm $p$-\textsc{MeanPortfolio} remain small while achieving a very high approximation factor. Across all three environments, the portfolio size never exceeds 10, yet still attains an approximation factor close to 1. This suggests that a small portfolio is sufficient to effectively cover the entire spectrum of $p$ values $\le 1$.


\textbf{Approximation Quality}. The proposed algorithms achieve significantly better approximation quality than both benchmark methods. In particular, $p$-\textsc{MeanPortfolio} generally attains the highest approximation quality, followed closely by \textsc{BudgetConstrainedPortfolio} in most cases. 

\edit{However, \textsc{BudgetConstrainedPortfolio} and random $p$ sampling occasionally outperform $p$-\textsc{MeanPortfolio}. This is because $p$-\textsc{MeanPortfolio} selects policies solely to meet the input approximation factor $\alpha$, but has no incentive or mechanism to surpass this approximation quality.

In contrast, \textsc{BudgetConstrainedPortfolio} fully utilizes the input budget $K$, and strategically selects boundary points first (a small initial $p$ followed by $p=1$). When $K$ is very small (1 or 2), this initial heuristic strategy can provide more effective coverage across $p \le 1$, explaining its occasional superior performance. Likewise, in rare cases where $K = 1$, a randomly chosen $p$ may happen to yield better coverage than the single policy selected by $p$-\textsc{MeanPortfolio}.}

\textbf{Computational Efficiency}.
As noted earlier, $p$-\textsc{MeanPortfolio} requires a large number of oracle calls to produce a high-quality portfolio. In contrast, \textsc{BudgetConstrainedPortfolio} achieves comparable quality while using significantly fewer oracle calls. This suggests that when computational resources are limited, \textsc{BudgetConstrainedPortfolio} serves as a strong alternative with good empirical performance.

\begin{table}[t]
\caption{A comparison of actual approximation ratios across various portfolio sizes for $p$-\textsc{MeanPortfolio}, random policy sampling, and random $p$ sampling. $p$-\textsc{MeanPortfolio} consistently outperforms the other two methods across all portfolio sizes and experiments.}
\label{tab: approximations}
\centering \footnotesize
\begin{tabular}{|c|c|c|c|}
\hline
\begin{tabular}[c]{@{}c@{}}Portfolio\\ Size\end{tabular} &
  \begin{tabular}[c]{@{}c@{}}$p$-\textsc{Mean}-\\ \textsc{Portfolio}\end{tabular} &
  \begin{tabular}[c]{@{}c@{}}Random\\ Policy\\ Sampling\end{tabular} &
  \begin{tabular}[c]{@{}c@{}}Random\\ $p$\\ Sampling\end{tabular} \\ \hline
\multicolumn{4}{c}{\rule{0pt}{10pt} Resource Allocation after Natural Disaster} \\ \hline
1 & 0.706 & 0.534 & \textbf{0.832} \\ \hline
2 & \textbf{0.904} & 0.568 & 0.890 \\ \hline
3 & \textbf{0.999} & 0.591 & 0.892 \\ \hline
4 & \textbf{1.000} & 0.609 & 0.888 \\ \hline
 \multicolumn{4}{c}{\rule{0pt}{10pt} Healthcare Intervention} \\ \hline
1 & \textbf{0.924} & 0.479 & 0.913 \\ \hline
2 & \textbf{0.982} & 0.545 & 0.941 \\ \hline
3 & \textbf{0.982} & 0.589 & 0.929 \\ \hline
4 & \textbf{0.982} & 0.625 & 0.947 \\ \hline
5 & \textbf{0.993} & 0.635 & 0.957 \\ \hline
6 & \textbf{0.999} & 0.647 & 0.962 \\ \hline
7 & \textbf{1.000} & 0.667 & 0.953 \\ \hline
\multicolumn{4}{c}{\rule{0pt}{10pt} Taxi Environment} \\ \hline
1  & \textbf{0.66} & 0.24 & \textbf{0.66} \\ \hline
2  & \textbf{0.61} & 0.31 & \textbf{0.61} \\ \hline
3  & \textbf{0.97} & 0.54 & 0.65          \\ \hline
8  & \textbf{0.87} & 0.71 & 0.67          \\ \hline
10 & \textbf{0.99} & 0.60 & 0.65          \\ \hline
\end{tabular}
\end{table}


\begin{table}[t]
\caption{A comparison of the number of oracle calls and actual approximation ratios for $p$-\textsc{MeanPortfolio} and \textsc{BudgetConstrainedPortfolio} across different portfolio sizes. While $p$-\textsc{MeanPortfolio} often achieves slightly better approximation, it requires significantly more oracle calls. \footnotemark}
\label{tab: oracle-call-comparison}
\centering\footnotesize
\hspace*{-1em}
\begin{tabular}{|c|cc|cc|}
\hline
\multirow{2}{*}{\begin{tabular}[c]{@{}c@{}}Portfolio\\ Size\end{tabular}} &
  \multicolumn{2}{c|}{$p$-\textsc{MeanPortfolio}} &
  \multicolumn{2}{c|}{\begin{tabular}[c]{@{}c@{}}\textsc{BudgetConstrained}-\\ \textsc{Portfolio}\end{tabular}} \\ \cline{2-5} 
 &
  \multicolumn{1}{c|}{\begin{tabular}[c]{@{}c@{}}Oracle\\ Calls\end{tabular}} &
  \begin{tabular}[c]{@{}c@{}}Actual\\ Approximation\end{tabular} &
  \multicolumn{1}{c|}{\begin{tabular}[c]{@{}c@{}}Oracle\\ Calls\end{tabular}} &
  \begin{tabular}[c]{@{}c@{}}Actual\\ Approximation\end{tabular} \\ \hline
  
\multicolumn{5}{c}{\rule{0pt}{10pt} Resource Allocation after Natural Disaster} \\ \hline
1 & \multicolumn{1}{c|}{1}  & 0.706 & \multicolumn{1}{c|}{\bf 1} & {\bf 0.885} \\ \hline
2 & \multicolumn{1}{c|}{7}  & 0.904 & \multicolumn{1}{c|}{\bf 2} & {\bf 0.921} \\ \hline
3 & \multicolumn{1}{c|}{18} & {\bf 0.999} & \multicolumn{1}{c|}{\bf 3} & 0.921 \\ \hline
4 & \multicolumn{1}{c|}{50} & {\bf 1.000} & \multicolumn{1}{c|}{\bf 4} & 0.921 \\ \hline

\multicolumn{5}{c}{\rule{0pt}{10pt} Healthcare Intervention} \\ \hline
1 & \multicolumn{1}{c|}{2}  & 0.924 & \multicolumn{1}{c|}{\bf 1} & {\bf 0.938} \\ \hline
2 & \multicolumn{1}{c|}{7}  & {\bf 0.982} & \multicolumn{1}{c|}{\bf 2} & 0.938 \\ \hline
3 & \multicolumn{1}{c|}{11} & {\bf 0.982} & \multicolumn{1}{c|}{\bf 3} & 0.938 \\ \hline
4 & \multicolumn{1}{c|}{19} & {\bf 0.982} & \multicolumn{1}{c|}{\bf 4} & 0.938 \\ \hline
5 & \multicolumn{1}{c|}{23} & {\bf 0.993} & \multicolumn{1}{c|}{\bf 5} & 0.986 \\ \hline
6 & \multicolumn{1}{c|}{46} & {\bf 0.999} & \multicolumn{1}{c|}{\bf 6} & 0.986 \\ \hline
7 & \multicolumn{1}{c|}{61} & {\bf 1.000} & \multicolumn{1}{c|}{\bf 7} & 0.993 \\ \hline

\multicolumn{5}{c}{\rule{0pt}{10pt} Taxi Environment} \\ \hline  
1  & \multicolumn{1}{c|}{17}  & {\bf 0.66} & \multicolumn{1}{c|}{\bf 1}  & {\bf 0.66} \\ \hline
2  & \multicolumn{1}{c|}{30}  & {\bf 0.61} & \multicolumn{1}{c|}{\bf 2}  & {\bf 0.61} \\ \hline
3  & \multicolumn{1}{c|}{44}  & {\bf 0.97} & \multicolumn{1}{c|}{\bf 3}  & 0.92 \\ \hline
8  & \multicolumn{1}{c|}{118} & 0.87  & \multicolumn{1}{c|}{\bf 8}  & {\bf 0.90} \\ \hline
10 & \multicolumn{1}{c|}{144} & {\bf 0.99} & \multicolumn{1}{c|}{\bf 10} & 0.92 \\ \hline
\end{tabular}
\end{table}

\footnotetext{Here, we observe a drop in the actual approximation factor of $p$-\textsc{MeanPortfolio} when the portfolio size is 8. Although this portfolio is computed using $\alpha = 0.9$, its achieved approximation is slightly lower. This discrepancy is likely due to the inherent randomness in the RL algorithm and the challenges of computing exact optimal policies in RL settings, as discussed in Section \ref{sec:alg}.}

\textbf{Diversity}.
As we have shown in Figure \ref{fig: sclm-education-age} in Section \ref{subsec:example}, portfolios can simplify policy deployment decisions by presenting a small yet diverse set of policies. \edit{We also provide the entire $p$ values chosen by $p$-\textsc{MeanPortfolio} in Appendix \ref{sec:additional_results}, Table \ref{tab: p-values}. We observe that the $p$ values are quite diverse, which corroborates the algorithm's objective to cover the entire $p$ range only with these selected values. Figures \ref{fig: pvalues_taxi} and \ref{fig: fraction-of-need-met-natural-disaster} further illustrate how the outcomes of the optimal policies for different $p$ values in the portfolio lead to varying impacts for the stakeholders.}  Importantly, these are not arbitrary policies but optimal policies for specific $p$ values, with approximation guarantees extending to other $p$s as well.

\section{Conclusion}
\label{sec:conclusion}

In this paper, we studied the concept of an $\alpha$-approximate portfolio in MORL for generalized $p$-means. We proposed an algorithm to compute $\alpha$-approximate portfolios and established theoretical guarantees on the trade-offs among $\alpha$, portfolio size, and the number of oracle calls. Additionally, we presented a theoretical analysis of warm-start techniques and developed an efficient heuristic algorithm. Our numerical experiments demonstrated that the proposed methods successfully compute compact portfolios that effectively capture the entire spectrum of $p$ values, empowering decision-makers to make informed decisions.

\newpage

\section*{Acknowledgments}
We thank NSF AI Institute for Societal Decision Making (AI-SDM) Award No. 2229881, ONR MURI N00014-24-1-2742, and NSF CAREER 2239824 for their support.

\section*{Impact Statement}
Our work studies the setting where a deployed RL policy impacts multiple stakeholders differently, a scenario that arises in various societal applications, including healthcare and LLMs. By summarizing the space of policies induced under different $p$-values, our proposed approach aims to assist decision-makers in making informed and equitable choices. Furthermore, this work has the potential to stimulate discussions within the machine learning community about the complexities and pluralities of social welfare notions, and the importance of addressing them. However, our focus is limited to $p$-means social welfare functions, which may not fully capture all dimensions of fairness in real-world scenarios. 

Our experiments use real-world healthcare data provided by the NGO ARMMAN. These experiments are strictly secondary analyses, with no real-world deployment of the proposed algorithm or baseline methods. All data usage complied with ARMMAN’s ethics board approvals, and data exchange and analysis followed their ethics review committee's guidelines, including privacy, informed consent, and anonymization. All of the analyses is conducted in close collaboration with ARMMAN. 

This work belongs to the broader area of algorithmic decision-making, where deployment decisions can have real-world implications. We emphasize the importance of not relying on this study in isolation. Responsible deployment requires comprehensive field testing and human oversight. Moreover, we acknowledge broader challenges inherent in algorithmic decision-making, including potential biases in data, inaccuracies in the underlying mathematical models, accountability issues, and concerns related to consent and privacy. Addressing these challenges is essential to ensuring that algorithmic decisions are ethically sound.

\bibliography{references_ver1}

\begin{thebibliography}{41}
\providecommand{\natexlab}[1]{#1}
\providecommand{\url}[1]{\texttt{#1}}
\expandafter\ifx\csname urlstyle\endcsname\relax
  \providecommand{\doi}[1]{doi: #1}\else
  \providecommand{\doi}{doi: \begingroup \urlstyle{rm}\Url}\fi

\bibitem[Agarwal et~al.(2022)Agarwal, Aggarwal, and Lan]{agarwal22}
Mridul Agarwal, Vaneet Aggarwal, and Tian Lan.
\newblock Multi-objective reinforcement learning with non-linear scalarization.
\newblock In \emph{Proceedings of the 21st International Conference on Autonomous Agents and Multiagent Systems}, AAMAS '22, page 9–17, Richland, SC, 2022. International Foundation for Autonomous Agents and Multiagent Systems.
\newblock ISBN 9781450392136.

\bibitem[Alamdari et~al.(2024)Alamdari, Ebadian, and Procaccia]{alamdari2024policy}
Parand~A. Alamdari, Soroush Ebadian, and Ariel~D. Procaccia.
\newblock Policy aggregation.
\newblock In \emph{The Thirty-eighth Annual Conference on Neural Information Processing Systems}, 2024.
\newblock URL \url{https://openreview.net/forum?id=ybiUVIxJth}.

\bibitem[Alegre et~al.(2023)Alegre, Roijers, Now{\'e}, Bazzan, and {da Silva}]{Alegre+2023}
Lucas~N. Alegre, Diederik~M. Roijers, Ann Now{\'e}, Ana L.~C. Bazzan, and Bruno~C. {da Silva}.
\newblock Sample-efficient multi-objective learning via generalized policy improvement prioritization.
\newblock In \emph{Proc. of the 22nd International Conference on Autonomous Agents and Multiagent Systems (AAMAS)}, 2023.

\bibitem[ARMMAN(2024)]{ARMMAN}
ARMMAN.
\newblock Armman: Advancing reduction in mortality and morbidity of mothers, children, and neonates, 2024.
\newblock URL \url{https://armman.org/}.

\bibitem[Bullen(2003)]{Bullen03}
P.S Bullen.
\newblock \emph{Handbook of Means and Their Inequalities}.
\newblock Mathematics and Its Applications ; 560. Springer Netherlands : Imprint: Springer, Dordrecht, 2nd ed. 2003. edition, 2003.

\bibitem[Chakrabarty and Swamy(2019)]{chakrabarty2019approximation}
Deeparnab Chakrabarty and Chaitanya Swamy.
\newblock Approximation algorithms for minimum norm and ordered optimization problems.
\newblock In \emph{Proceedings of the 51st Annual ACM SIGACT Symposium on Theory of Computing}, pages 126--137, 2019.

\bibitem[Chakraborty et~al.(2024)Chakraborty, Qiu, Yuan, Koppel, Manocha, Huang, Bedi, and Wang]{Chakraborty24}
Souradip Chakraborty, Jiahao Qiu, Hui Yuan, Alec Koppel, Dinesh Manocha, Furong Huang, Amrit Bedi, and Mengdi Wang.
\newblock {M}ax{M}in-{RLHF}: Alignment with diverse human preferences.
\newblock In Ruslan Salakhutdinov, Zico Kolter, Katherine Heller, Adrian Weller, Nuria Oliver, Jonathan Scarlett, and Felix Berkenkamp, editors, \emph{Proceedings of the 41st International Conference on Machine Learning}, volume 235 of \emph{Proceedings of Machine Learning Research}, pages 6116--6135. PMLR, 21--27 Jul 2024.
\newblock URL \url{https://proceedings.mlr.press/v235/chakraborty24b.html}.

\bibitem[Chen et~al.(2021)Chen, Wang, and Lan]{chen21}
Jingdi Chen, Yimeng Wang, and Tian Lan.
\newblock Bringing fairness to actor-critic reinforcement learning for network utility optimization.
\newblock In \emph{IEEE INFOCOM 2021 - IEEE Conference on Computer Communications}, pages 1--10, 2021.
\newblock \doi{10.1109/INFOCOM42981.2021.9488823}.

\bibitem[Cousins(2023)]{cousins23}
Cyrus Cousins.
\newblock Revisiting fair-pac learning and the axioms of cardinal welfare.
\newblock In Francisco Ruiz, Jennifer Dy, and Jan-Willem van~de Meent, editors, \emph{Proceedings of The 26th International Conference on Artificial Intelligence and Statistics}, volume 206 of \emph{Proceedings of Machine Learning Research}, pages 6422--6442. PMLR, 25--27 Apr 2023.
\newblock URL \url{https://proceedings.mlr.press/v206/cousins23a.html}.

\bibitem[Cousins et~al.(2024)Cousins, Asadi, Lobo, and Littman]{cousins2024welfare}
Cyrus Cousins, Kavosh Asadi, Elita Lobo, and Michael Littman.
\newblock On welfare-centric fair reinforcement learning.
\newblock \emph{Reinforcement Learning Journal}, 3:\penalty0 1124--1137, 2024.

\bibitem[Dietterich(2000)]{dietterich2000hierarchical}
Thomas~G Dietterich.
\newblock Hierarchical reinforcement learning with the maxq value function decomposition.
\newblock \emph{Journal of artificial intelligence research}, 13:\penalty0 227--303, 2000.

\bibitem[Drygala et~al.(2024)Drygala, Lattanzi, Maggiori, Stouras, Svensson, and Vassilvitskii]{drygala2024data}
Marina Drygala, Silvio Lattanzi, Andreas Maggiori, Miltiadis Stouras, Ola Svensson, and Sergei Vassilvitskii.
\newblock Data-driven solution portfolios.
\newblock \emph{arXiv preprint arXiv:2412.00717}, 2024.

\bibitem[Fan et~al.(2022)Fan, Peng, Tian, and Fain]{fan2022welfare}
Zimeng Fan, Nianli Peng, Muhang Tian, and Brandon Fain.
\newblock Welfare and fairness in multi-objective reinforcement learning.
\newblock \emph{arXiv preprint arXiv:2212.01382}, 2022.

\bibitem[Fan et~al.(2023)Fan, Peng, Tian, and Fain]{fan23}
Ziming Fan, Nianli Peng, Muhang Tian, and Brandon Fain.
\newblock Welfare and fairness in multi-objective reinforcement learning.
\newblock In \emph{Proceedings of the 2023 International Conference on Autonomous Agents and Multiagent Systems}, AAMAS '23, page 1991–1999, Richland, SC, 2023. International Foundation for Autonomous Agents and Multiagent Systems.
\newblock ISBN 9781450394321.

\bibitem[Goel and Meyerson(2006)]{goel2006simultaneous}
Ashish Goel and Adam Meyerson.
\newblock Simultaneous optimization via approximate majorization for concave profits or convex costs.
\newblock \emph{Algorithmica}, 44:\penalty0 301--323, 2006.

\bibitem[Golovin et~al.(2008)Golovin, Gupta, Kumar, and Tangwongsan]{golovin2008all}
Daniel Golovin, Anupam Gupta, Amit Kumar, and Kanat Tangwongsan.
\newblock All-norms and all-l\_p-norms approximation algorithms.
\newblock In \emph{IARCS Annual Conference on Foundations of Software Technology and Theoretical Computer Science (2008)}. Schloss-Dagstuhl-Leibniz Zentrum f{\"u}r Informatik, 2008.

\bibitem[Gupta et~al.(2023)Gupta, Moondra, and Singh]{gupta2024a}
Swati Gupta, Jai Moondra, and Mohit Singh.
\newblock Which lp norm is the fairest? {A}pproximations for fair facility location across all "p".
\newblock In \emph{Proceedings of the 24th ACM Conference on Economics and Computation}, EC '23, page 817, New York, NY, USA, 2023. Association for Computing Machinery.
\newblock ISBN 9798400701047.
\newblock \doi{10.1145/3580507.3597664}.
\newblock URL \url{https://doi.org/10.1145/3580507.3597664}.

\bibitem[Gupta et~al.(2025)Gupta, Moondra, and Singh]{gupta2024b}
Swati Gupta, Jai Moondra, and Mohit Singh.
\newblock Balancing notions of equity: Trade-offs between fair portfolio sizes and achievable guarantees.
\newblock In \emph{Proceedings of the 2025 Annual ACM-SIAM Symposium on Discrete Algorithms (SODA)}, 2025.
\newblock URL \url{https://arxiv.org/abs/2311.03230}.

\bibitem[Hao et~al.(2023)Hao, Xu, Zhang, Yang, and Muntean]{hao23}
Hao Hao, Changqiao Xu, Wei Zhang, Shujie Yang, and Gabriel-Miro Muntean.
\newblock Computing offloading with fairness guarantee: A deep reinforcement learning method.
\newblock \emph{IEEE Transactions on Circuits and Systems for Video Technology}, 33\penalty0 (10):\penalty0 6117--6130, 2023.
\newblock \doi{10.1109/TCSVT.2023.3255229}.

\bibitem[Hayes et~al.(2022)Hayes, R{\u a}dulescu, Bargiacchi, K{\"a}llstr{\"o}m, Macfarlane, Reymond, Verstraeten, Zintgraf, Dazeley, Heintz, Howley, Irissappane, Mannion, Now{\'e}, Ramos, Restelli, Vamplew, and Roijers]{hayes_practical_2022}
Conor~F. Hayes, Roxana R{\u a}dulescu, Eugenio Bargiacchi, Johan K{\"a}llstr{\"o}m, Matthew Macfarlane, Mathieu Reymond, Timothy Verstraeten, Luisa~M. Zintgraf, Richard Dazeley, Fredrik Heintz, Enda Howley, Athirai~A. Irissappane, Patrick Mannion, Ann Now{\'e}, Gabriel Ramos, Marcello Restelli, Peter Vamplew, and Diederik~M. Roijers.
\newblock A practical guide to multi-objective reinforcement learning and planning.
\newblock \emph{Autonomous Agents and Multi-Agent Systems}, 36\penalty0 (1):\penalty0 26, April 2022.

\bibitem[Ju et~al.(2024)Ju, Ghosh, and Shroff]{ju2024achieving}
Peizhong Ju, Arnob Ghosh, and Ness Shroff.
\newblock Achieving fairness in multi-agent {MDP} using reinforcement learning.
\newblock In \emph{The Twelfth International Conference on Learning Representations}, 2024.
\newblock URL \url{https://openreview.net/forum?id=yoVq2BGQdP}.

\bibitem[Mandal and Gan(2023)]{mandal2023sociallyfairreinforcementlearning}
Debmalya Mandal and Jiarui Gan.
\newblock Socially fair reinforcement learning, 2023.
\newblock URL \url{https://arxiv.org/abs/2208.12584}.

\bibitem[Moulin(2003)]{moulin_fair_2003}
Herv{\'e} Moulin.
\newblock \emph{Fair {Division} and {Collective} {Welfare}}.
\newblock The MIT Press, January 2003.

\bibitem[Papadimitriou and Tsitsiklis(1999)]{papadimitriou1999complexity}
Christos~H Papadimitriou and John~N Tsitsiklis.
\newblock The complexity of optimal queuing network control.
\newblock \emph{Mathematics of Operations Research}, 24\penalty0 (2):\penalty0 293--305, 1999.

\bibitem[Pardeshi et~al.(2024)Pardeshi, Shapira, Procaccia, and Singh]{pardeshi2024learning}
Kanad~Shrikar Pardeshi, Itai Shapira, Ariel~D. Procaccia, and Aarti Singh.
\newblock Learning social welfare functions.
\newblock In \emph{The Thirty-eighth Annual Conference on Neural Information Processing Systems}, 2024.
\newblock URL \url{https://openreview.net/forum?id=7O6KtaAr8n}.

\bibitem[Parisi et~al.(2014)Parisi, Pirotta, Smacchia, Bascetta, and Restelli]{parisi14}
Simone Parisi, Matteo Pirotta, Nicola Smacchia, Luca Bascetta, and Marcello Restelli.
\newblock Policy gradient approaches for multi-objective sequential decision making.
\newblock In \emph{2014 International Joint Conference on Neural Networks (IJCNN)}, pages 2323--2330, 2014.
\newblock \doi{10.1109/IJCNN.2014.6889738}.

\bibitem[Park et~al.(2024)Park, Liu, Kong, Zhang, and Ozdaglar]{park2024rlhf}
Chanwoo Park, Mingyang Liu, Dingwen Kong, Kaiqing Zhang, and Asuman Ozdaglar.
\newblock {RLHF} from heterogeneous feedback via personalization and preference aggregation, 2024.
\newblock URL \url{https://arxiv.org/abs/2405.00254}.

\bibitem[Perez et~al.(2009)Perez, Germain-Renaud, K\'{e}gl, and Loomis]{perez09}
Julien Perez, C\'{e}cile Germain-Renaud, Bal\'{a}zs K\'{e}gl, and Charles Loomis.
\newblock Responsive elastic computing.
\newblock In \emph{Proceedings of the 6th International Conference Industry Session on Grids Meets Autonomic Computing}, GMAC '09, page 55–64, New York, NY, USA, 2009. Association for Computing Machinery.
\newblock ISBN 9781605585789.
\newblock \doi{10.1145/1555301.1555311}.
\newblock URL \url{https://doi.org/10.1145/1555301.1555311}.

\bibitem[R{\u a}dulescu et~al.(2019)R{\u a}dulescu, Mannion, Roijers, and Now{\'e}]{radulescu_multi-objective_2019}
Roxana R{\u a}dulescu, Patrick Mannion, Diederik~M. Roijers, and Ann Now{\'e}.
\newblock Multi-objective multi-agent decision making: a utility-based analysis and survey.
\newblock \emph{Autonomous Agents and Multi-Agent Systems}, 34\penalty0 (1):\penalty0 10, December 2019.

\bibitem[Reymond et~al.(2022)Reymond, Bargiacchi, and Now\'{e}]{reymond22}
Mathieu Reymond, Eugenio Bargiacchi, and Ann Now\'{e}.
\newblock Pareto conditioned networks.
\newblock In \emph{Proceedings of the 21st International Conference on Autonomous Agents and Multiagent Systems}, AAMAS '22, page 1110–1118, Richland, SC, 2022. International Foundation for Autonomous Agents and Multiagent Systems.
\newblock ISBN 9781450392136.

\bibitem[Roberts(1980)]{roberts_interpersonal_1980}
Kevin W.~S. Roberts.
\newblock Interpersonal {Comparability} and {Social} {Choice} {Theory}.
\newblock \emph{The Review of Economic Studies}, 47\penalty0 (2):\penalty0 421--439, January 1980.

\bibitem[Roijers et~al.(2013)Roijers, Vamplew, Whiteson, and Dazeley]{roijers13}
Diederik~M. Roijers, Peter Vamplew, Shimon Whiteson, and Richard Dazeley.
\newblock A survey of multi-objective sequential decision-making.
\newblock \emph{J. Artif. Int. Res.}, 48\penalty0 (1):\penalty0 67–113, October 2013.
\newblock ISSN 1076-9757.

\bibitem[Van~Moffaert and Now\'{e}(2014)]{moff14}
Kristof Van~Moffaert and Ann Now\'{e}.
\newblock Multi-objective reinforcement learning using sets of pareto dominating policies.
\newblock \emph{J. Mach. Learn. Res.}, 15\penalty0 (1):\penalty0 3483–3512, January 2014.
\newblock ISSN 1532-4435.

\bibitem[Verma et~al.(2023)Verma, Singh, Mate, Verma, Gorantla, Madhiwalla, Hegde, Thakkar, Jain, Tambe, et~al.]{verma2023expanding}
Shresth Verma, Gargi Singh, Aditya Mate, Paritosh Verma, Sruthi Gorantla, Neha Madhiwalla, Aparna Hegde, Divy Thakkar, Manish Jain, Milind Tambe, et~al.
\newblock Expanding impact of mobile health programs: Saheli for maternal and child care.
\newblock \emph{AI Magazine}, 44\penalty0 (4):\penalty0 363--376, 2023.

\bibitem[Verma et~al.(2024{\natexlab{a}})Verma, Boehmer, Kong, and Tambe]{verma2024}
Shresth Verma, Niclas Boehmer, Lingkai Kong, and Milind Tambe.
\newblock Balancing act: {P}rioritization strategies for {LLM}-designed restless bandit rewards, 2024{\natexlab{a}}.
\newblock URL \url{https://arxiv.org/abs/2408.12112}.

\bibitem[Verma et~al.(2024{\natexlab{b}})Verma, Singh, Mate, Verma, Gorantla, Madhiwalla, Hegde, Thakkar, Jain, Tambe, and Taneja]{Verma2024b}
Shresth Verma, Gargi Singh, Aditya Mate, Paritosh Verma, Sruthi Gorantla, Neha Madhiwalla, Aparna Hegde, Divy Thakkar, Manish Jain, Milind Tambe, and Aparna Taneja.
\newblock Increasing impact of mobile health programs: Saheli for maternal and child care.
\newblock \emph{Proceedings of the AAAI Conference on Artificial Intelligence}, 37\penalty0 (13):\penalty0 15594--15602, Jul. 2024{\natexlab{b}}.

\bibitem[Whittle(1988)]{whittle98}
P.~Whittle.
\newblock Restless bandits: Activity allocation in a changing world.
\newblock \emph{Journal of Applied Probability}, 25:\penalty0 287--298, 1988.
\newblock ISSN 00219002.
\newblock URL \url{http://www.jstor.org/stable/3214163}.

\bibitem[Wu et~al.(2018)Wu, Zeng, and Zhang]{wu18}
Qingqing Wu, Yong Zeng, and Rui Zhang.
\newblock Joint trajectory and communication design for multi-uav enabled wireless networks.
\newblock \emph{IEEE Transactions on Wireless Communications}, 17\penalty0 (3):\penalty0 2109--2121, 2018.
\newblock \doi{10.1109/TWC.2017.2789293}.

\bibitem[Yang et~al.(2019)Yang, Sun, and Narasimhan]{yang19}
Runzhe Yang, Xingyuan Sun, and Karthik Narasimhan.
\newblock \emph{A generalized algorithm for multi-objective reinforcement learning and policy adaptation}.
\newblock Curran Associates Inc., Red Hook, NY, USA, 2019.

\bibitem[Yu et~al.(2024)Yu, Siddique, and Weng]{Yu24}
Guanbao Yu, Umer Siddique, and Paul Weng.
\newblock Fair deep reinforcement learning with generalized gini welfare functions.
\newblock In Francesco Amigoni and Arunesh Sinha, editors, \emph{Autonomous Agents and Multiagent Systems. Best and Visionary Papers}, pages 3--29, Cham, 2024. Springer Nature Switzerland.

\bibitem[Zhong et~al.(2024)Zhong, Deng, Su, Wu, and Zhang]{zhong2024rlhf}
Huiying Zhong, Zhun Deng, Weijie~J. Su, Zhiwei~Steven Wu, and Linjun Zhang.
\newblock Provable multi-party reinforcement learning with diverse human feedback, 2024.
\newblock URL \url{https://arxiv.org/abs/2403.05006}.

\end{thebibliography}
\bibliographystyle{plainnat}

\newpage
\onecolumn

\appendix

\section{Omitted Proofs}\label{sec: omitted-proofs}

We include omitted proofs and various lemmas here.

\subsection{Proof of Theorem \ref{thm: portfolio-with-line-search}}

Our first lemma shows establishes the monotonicity of aggregation functions $v^{(\ell)}$ in $p$:

\begin{lemma}\label{lem: value-functions-are-monotone}
    For any fixed policy $\pi \in \Pi$, $v^{(\ell)}(\pi, p)$ is monotone increasing in $p$ for $\ell \in \{1, 2\}$.
\end{lemma}

\begin{proof}
    {$\ell = 1$:} For any $\tau \sim \pi$, Lemma \ref{lem: social-welfare-monotonicity} implies that $f(\mathbf{G}(\tau), p)$ is monotone increasing in $p$. Therefore,
    \[
        v^{(1)}(\pi, p) = \mathbb{E}_{\tau \sim \pi} [f(\mathbf{G}(\tau), p)]
    \]
    is monotone increasing in $p$.

    {$\ell = 2$:} Denote $\mathbf{x} = \mathbb{E}_{\tau \sim \pi}[\mathbf{G}(\tau)]$. Then the monotonicity of $v^{(2)}(\pi, p) = f(\mathbf{x}, p)$ follows directly from Lemma \ref{lem: social-welfare-monotonicity}.
\end{proof}

The following corollary follows immediately:

\begin{corollary}\label{cor: agg-function-is-monotone-increasing}
    $v_*^{(\ell)}(p) := \max_{\pi \in \Pi} v^{(\ell)}(\pi, p)$ is monotone increasing in $p$ for $\ell \in \{1, 2\}$.
\end{corollary}

The following two lemmas show that the $p$-mean for $p = - \infty$ is $\alpha$-approximated by the $p_0$-mean where $p_0 = - \frac{\ln N}{\ln(1/\alpha)}$, so we can effectively restrict to $[-p_0, 1]$ when finding portfolios:

\begin{lemma}\label{lem: approximation-min-with-large-negative-p}
    Given a vector $\mathbf{x} \in \R_{> 0}^N$ and $\alpha \in (0, 1)$, define $p_0 = - \frac{\ln N}{\ln (1/\alpha)}$. Then,
    \[
        f(\mathbf{x}, -\infty) \ge \alpha \cdot f(\mathbf{x}, p) \quad \forall \ p \le p_0.
    \]
\end{lemma}

\begin{proof}
    Suppose $0 < x_1 \le \ldots \le x_n$, so that $f(\mathbf{x}, -\infty) = \min_{j \in [N]} x_j = x_1$. Given $p \le p_0$, denote $q = - p \ge \frac{\ln N}{\ln(1/\alpha)}$. Then, since
    \[
        f(\mathbf{x}, p) = \left(\frac{1}{N} \sum_{j \in [N]} x_j^p\right)^{1/p} = \frac{1}{\left(\frac{1}{N} \sum_{j \in [N]} \frac{1}{x_j^{q}}\right)^{1/q}} = \frac{1}{\frac{1}{x_1}\left(\frac{1}{N} \sum_{j \in [N]} \left(\frac{x_1}{x_j}\right)^q\right)^{1/q}},
    \]
    we get
    \[
        \frac{1}{f(\mathbf{x}, p)} \ge \frac{1}{x_1} \left(\frac{1}{N} \times \left(\frac{x_1}{x_1}\right)^q\right)^{1/q} = \frac{N^{1/p}}{x_1} \ge \frac{N^{1/p_0}}{x_1} = \frac{\alpha}{f(\mathbf{x}, -\infty)},
    \]
    or that $f(\mathbf{x}, -\infty) \ge \alpha \cdot f(\mathbf{x}, p)$.
\end{proof}

\begin{lemma}\label{lem: approximation-min-with-large-negative-p-agg}
    Given an MDP $\mathcal{M}$ with $N$ reward functions, set $\Pi$ of feasible policies, and an approximation factor $\alpha \in (0, 1)$, denote $p_0 = - \frac{\ln N}{\ln(1/\alpha)}$. Then, for a given aggregation rule $\ell \in \{1, 2\}$, policy $\pi_0 = {\arg\max}_{\pi \in \Pi} v^{(\ell)}(\pi, p_0)$ is an $\alpha$-approximation for all $p \le p_0$. That is,
    \[
        v^{(\ell)}(\pi_0, p) \ge v_*^{(\ell)}(p) = \max_{\pi \in \Pi} v^{(\ell)}(\pi, p) \quad \forall \ p \le p_0.
    \]
\end{lemma}

\begin{proof}
    $\ell = 1$: From Lemma \ref{lem: approximation-min-with-large-negative-p}, for all $p \le p_0$, we get
    \begin{align*}
        v^{(1)}(\pi_0, p) &\ge v^{(1)}(\pi_0, -\infty) & (\text{Lemma} \ \ref{lem: value-functions-are-monotone}) \\
        &= \mathbb{E}_{\tau \sim \pi} [f(\mathbf{G}(\tau), -\infty)] & (\text{eqn.} (\ref{eqn: first-aggregation-function})) \\
        &\ge \alpha \cdot \mathbb{E}_{\tau \sim \pi} [f(\mathbf{G}(\tau), p_0)] & (\text{Lemma} \ \ref{lem: approximation-min-with-large-negative-p}) \\
        &= \alpha \cdot v^{(1)}(\pi_0, p_0) & (\text{eqn.} (\ref{eqn: first-aggregation-function}))\\
        &= \alpha \cdot \max_{\pi \in \Pi} \cdot v^{(1)}(\pi, p_0)  \\
        &\ge \alpha \cdot \max_{\pi \in \Pi} v^{(1)}(\pi, p) & (\text{Lemma} \ \ref{lem: value-functions-are-monotone}).
    \end{align*}
    The result for $\ell = 2$ follows similarly and is omitted.
\end{proof}

The following three lemmas establish guarantees on \textsc{LineSearch}:

\begin{lemma}\label{lem: line-search-guarantee}
    Suppose \textsc{LineSearch} (Algorithm \ref{alg: line-search}) on input $v^{(\ell)}, p, \Pi, \alpha$ returns $b^* \in [p, 1]$. Then, either $b^* = 1$, or the policy $\pi := {\arg\max}_{\pi \in \Pi} v^{(\ell)}(\pi, p)$ satisfies
    \begin{align*}
        \frac{v^{(\ell)}_*(p)}{\sqrt{\alpha}} \le v_*^{(\ell)}(b^*) \le \frac{v^{(\ell)}(\pi, b^*)}{\alpha}.
    \end{align*}
\end{lemma}

\begin{proof}
    If $\textsc{LineSearch}(v^{(\ell)}, p, \Pi, \alpha)$ does not return $b^* = 1$, then we must have that $p$-\textsc{MeanPortfolio} enters the while loop (step \ref{step: main-algorithm-while-loop}) at least once. In particular,
    \[
        v_*^{(\ell)}(p) = v^{(\ell)}(\pi, p) < \alpha \ v_*^{(\ell)}(b^*) < \sqrt{\alpha} \ v_*^{(\ell)}(b^*).
    \]
    This proves the first inequality. For the second inequality, notice that the algorithm terminates only when $v^{(\ell)}(\pi, a) \ge \alpha \ v_*^{(\ell)}(b^*)$. As can be easily checked, the algorithm maintains the invariant $a \ge p$. Therefore, by Lemma \ref{lem: value-functions-are-monotone},
    \[
        v^{(\ell)}(\pi, b^*) \ge v^{(\ell)}(\pi, a) \ge \alpha \ v_*^{(\ell)}(b^*). \qedhere
    \]
\end{proof}

\begin{lemma}\label{lem: line-search-invariant}
    Algorithm \textsc{LineSearch} maintains the following invariant at all times: $v^{(\ell)}(\pi, a) \ge \sqrt{\alpha} \ v^{(\ell)}_*(a)$.
\end{lemma}

\begin{proof}
    Initially, $p = a$, and therefore $\pi := {\arg\max}_{\pi' \in \Pi} v^{(\ell)}(\pi', p)$ satisfies $v^{(\ell)}(\pi, p) = v_*^{(\ell)}(p) \ge \sqrt{\alpha} \ v_*^{(\ell)}(p)$ since $\alpha \in (0, 1)$.

    Further, the algorithm only updates $a \gets q$ in step \ref{step: line-search-update-lower-bound} when $v^{(\ell)}(\pi, a) \ge \sqrt{\alpha} \ v_*^{(\ell)}(q)$. Since $q \ge a$, we get from Lemma \ref{lem: value-functions-are-monotone} that $v^{(\ell)}(\pi, q) \ge v^{(\ell)}(\pi, a)$, thus finishing the proof.
\end{proof}

\begin{lemma}\label{lem: line-search-approximation-guarantee}
    Algorithm $\textsc{LineSearch}$ on input $p, v^{(\ell)}, \Pi, \alpha$ returns $b^* > p$ such that $\pi := {\arg\max}_{\pi' \in \Pi} v^{(\ell)}(\pi, p)$ is an $\alpha$-approximation for all $q \in [p, b^*]$, i.e., $v^{(\ell)}(\pi, q) \ge \alpha \ v_*^{(\ell)}(q)$ for all such $q$.
\end{lemma}

\begin{proof}
    \textsc{LineSearch} starts with $a \gets p$ and keeps increasing $a \gets q$ whenever $v^{(\ell)}(\pi, a) \ge \sqrt{\alpha} \ v^{(\ell)}_*(q)$ for $q = \frac{a + b}{2}$. In particular, whenever $a$ is increased, we get that for all $q' \in [a, (a + b)/2]$, from Lemma \ref{lem: value-functions-are-monotone},
    \[
        v^{(\ell)}(\pi, q') \ge v^{(\ell)}(\pi, a) \ge \sqrt{\alpha} \ v^{(\ell)}_*(\pi, (a + b)/2) \ge \sqrt{\alpha} \ v^{(\ell)}_*(\pi, q') > \alpha \ v^{(\ell)}_*(q').
    \]
    That is, at all times, the algorithm maintains the invariant that $\pi$ is an $\alpha$-approximation for all $q' \in [p, a]$. Further, the algorithm terminates at $b = b^*$ when $v^{(\ell)}(\pi, a) \ge v_*^{(\ell)}(b^*)$, i.e., for all $q' \in [a, b^*]$, we get
    \[
        v^{(\ell)}(\pi, q') \ge v^{(\ell)}(\pi, a) \ge \alpha \ v^{(\ell)}_*(b^*) \ge \alpha \ v^{(\ell)}_*(q').
    \]
\end{proof}

The next lemma bounds the slope of the logarithm of the $p$-mean function and consequently the slope of $\ln v^{(\ell)}_*(p)$:

\begin{lemma}\label{lem: slope-bound}
    \begin{enumerate}
        \item For any $\mathbf{x} \in \R_{> 0}^N$,  such that $HL \le x_i \le HU$ for all $i \in [N]$, define $g(\mathbf{x}, p) := \ln f(\mathbf{x}, p) = \frac{1}{p} \ln \left(\frac{1}{N} \sum_{i \in [N]} x_i^p\right)$. Then, if $L \le x_i \le U$ with condition number $\kappa := \frac{U}{L}$, we have
        \[
            \frac{dg(\mathbf{x}, p)}{dp} \le \kappa \ln \kappa.
        \]
        \item For any given policy $\pi \in \Pi$ and aggregation rule $\ell \in \{1, 2\}$, we have $\frac{d \left(\ln v^{(\ell)}(\pi, p)\right)}{dp} \le \kappa \ln \kappa$, where $\kappa$ is the condition number defined in Assumption \ref{assumption:bound}.
        \item For a given aggregation rule $\ell \in \{1, 2\}$, define $w(p) := \ln v^{(\ell)}_*(p)$. Then, for all distinct $p, q \le 1$,
        \[
            \frac{w(q) - w(p)}{q - p} \le \kappa \ln \kappa.
        \]
    \end{enumerate}
\end{lemma}

\begin{proof} \textbf{Part 1.}
    Note that $g(\beta \mathbf{x}, p) = \ln \beta + g(\mathbf{x}, p)$ for all $\mathbf{x} \in \R^N_{> 0}$, $p \le 1$, and $\beta > 0$. Therefore, we can assume without loss of generality that $LH = 1$ and $UH = \frac{UH}{LH} = \kappa$.

    That is, assume without of generality that $1 \le x_1 \le \ldots \le x_N \le \kappa$. Since $g(\mathbf{x}, p) = \frac{-\ln N}{p} + \frac{1}{p} \ln \left(\sum_{i \in [N]} x_i^p\right)$, we get that $p g(\mathbf{x}, p) = - \ln N + \ln \left(\sum_{i \in [N]} x_i^p\right)$. Therefore,
    \begin{equation}\label{eqn: log-p-mean-derivative-1}
        g(\mathbf{x}, p) + p \frac{dg(\mathbf{x}, p)}{dp} = \frac{\sum_i (\ln x) \cdot x_i^p}{\sum_{i} x_i^p} .
    \end{equation}
    Since $\ln t$ is concave in $t$, we get that
    \[
        p g(\mathbf{x}, p) = \ln \left(\frac{1}{N} \sum_{i} x_i^p\right) \ge \frac{1}{N} \sum_{i} \ln x_i^p = \frac{p}{N} \sum_i \ln x_i.
    \]
    Therefore, we get
    \begin{equation}\label{eqn: log-p-mean-concavity-bounds}
        g(\mathbf{x}, p) \begin{cases}
            \ge \frac{1}{N} \sum_i \ln x_i & \text{if} \ p \ge 0, \\
            \le \frac{1}{N} \sum_i \ln x_i & \text{if} \ p < 0.
        \end{cases}
    \end{equation}
    Denote $\mu_p = \frac{1}{N} \sum_{i \in [N]} x_i^p$.
    
    \paragraph{Case A: $p \in [0, 1]$.} Plugging this back into eqn. (\ref{eqn: log-p-mean-derivative-1}), we get
    \[
        p \frac{dg(\mathbf{x}, p)}{dp} \le \frac{\sum_i (\ln x_i) \cdot x_i^p}{N \mu_p} - \frac{1}{N} \sum_i \ln x_i = \frac{1}{N} \left(\sum_{i} \ln x_i \left(\frac{x_i^p}{\mu_p} - 1\right)\right).
    \]
    Since $1 \le x_i \le \kappa$ for all $i$, we have $0 \le \ln x_i \le \ln \kappa$ for all $i$. Further, since $\mathbf{x}_N \ge x_i$, we get that
    \begin{equation}\label{eqn: log-p-mean-derivative-pos-2}
        Np \ \frac{dg(\mathbf{x}, p)}{dp} \le \sum_{i} \ln x_i \left(\frac{x_i^p}{\mu_p} - 1\right) \le \ln \kappa \sum_{i: x_i^p \ge \mu_p} \left(\frac{x_N^p}{\mu_p} - 1\right) \le N \ln \kappa \left(\frac{x_N^p}{\mu_p} - 1\right).
    \end{equation}
    Now, since $\mu_p$ is the mean of $x_i^p, i \in [N]$, we must have that $\mu_p = \theta^p$ for some $1 \le \theta \le \kappa$. Substituting this, we get $Np \ \frac{dg(\mathbf{x}, p)}{dp} \le N (\ln \kappa) \left((x_N/\theta)^p - 1\right)$. Since $\frac{x_N}{\theta} \le \frac{\kappa}{1} = \kappa$ and the derivative of $\alpha^p$ with respect to $\alpha$ is $p \alpha^{p - 1}$, we get
    \begin{align*}
        \left((x_N/\theta)^p - 1\right) &\le (\kappa^p - 1) \\
        &= \int_{1}^\kappa p \cdot \alpha^{p - 1} \ \text{d}\alpha \\
        &\le \int_1^\kappa p \cdot 1^{p - 1} \ \text{d}\alpha & (\text{since} \ p - 1 \le 0 \ \text{and so} \ \alpha^{p - 1} \ \text{is nonincreasing in } p)  \\ 
        &\le p (\kappa - 1) \le p\kappa. & (\text{since} \ p \ge 0)
    \end{align*}
    Plugging this back into eqn. (\ref{eqn: log-p-mean-derivative-pos-2}), we get
    \[
        p \frac{dg(\mathbf{x}, p)}{dp} \le p \kappa \ln \kappa, 
    \]
    or that $\frac{dg(\mathbf{x}, p)}{dp} \le \kappa \ln \kappa$ for all $p \in [0, 1]$.

    \paragraph{Case B: $p < 0$.} As before, plugging bound (\ref{eqn: log-p-mean-concavity-bounds}) into eqn. (\ref{eqn: log-p-mean-derivative-1}), we get
    \[
        p \frac{dg(\mathbf{x}, p)}{dp} \ge \frac{\sum_i (\ln x_i) \cdot x_i^p}{N\mu_p} - \frac{1}{N} \sum_i \ln x_i = \frac{1}{N} \left(\sum_{i} \ln x_i \left(\frac{x_i^p}{\mu_p} - 1\right)\right) = - \frac{1}{N} \left(\sum_{i} \ln x_i \left(1 - \frac{x_i^p}{\mu_p}\right)\right).
    \]
    That is, since $p < 0$,
    \[
        \frac{dg(\mathbf{x}, p)}{dp} \le \frac{1}{(-p)N} \left(\sum_{i} \ln x_i \left(1 - \frac{x_i^p}{\mu_p}\right)\right) \le \frac{\ln R}{(-p)N} \sum_{i: x_i^p \le \mu_p} \left(1 - \frac{x_i^p}{\mu_p}\right) \le \frac{\ln R}{(-p)N} \times N \left(1 - \frac
        {x_N^p}{\mu_p}\right) = \frac{\ln \kappa}{(-p)} \left(1 - \frac
        {x_N^p}{\mu_p}\right).
    \]
    Denote $- p = q$; then $q > 0$. As before, $\mu_p = \theta^p$ for some $1 \le \theta \le N$, and so we have
    \[
        \frac{dg(\mathbf{x}, p)}{dp} \le \frac{\ln \kappa}{q} \left(1 - \frac
        {\theta^q}{x_N^q}\right) \le \frac{\ln \kappa}{q} \left(1 - \frac{1}{\kappa^q}\right).
    \]
    For $q \ge 1$ (i.e., for $p \le -1$), this is at most $\ln \kappa$. For $q \in [0, 1]$, we will bound this in a manner similar to Case A:
    \begin{align*}
        \frac{dg(\mathbf{x}, p)}{dp} &\le \frac{\ln \kappa}{q} \int_{1/\kappa}^1 q \alpha^{q - 1} \ \text{d}\alpha \\
        &\le \ln \kappa \int_{1/\kappa}^1 \frac{1}{\kappa^{q - 1}} \ \text{d} \alpha \\
        &\le (\ln \kappa) \times \kappa^{1 - q} \le \kappa \ln \kappa.
    \end{align*}

    \textbf{Part 2.} $\ell = 1$. Given any trajectory $\tau$, for the vector $\mathbf{G}(\tau) \in [LH, UH]^N$, we get from Part 1 that
    \[
        \frac{d\left(\ln f(\mathbf{G}(\tau), p)\right)}{dp} \le \kappa \ln \kappa.
    \]
    Consequently, for any policy $\pi$, we get for $v^{(1)}(\pi, p) = \mathbb{E}_{\tau \sim \pi} \left[f(\mathbf{G}(\tau), p)\right]$ using linearity of expectation that
    \[
        \frac{d\left(\ln v^{(1)}(\pi, p)\right)}{p} = \mathbb{E}_{\tau \sim \pi} \left[\frac{d\left(\ln f(\mathbf{G}(\tau), p)\right)}{dp}\right] \le \mathbb{E}_{\tau \sim \pi} \left[\kappa \ln \kappa\right] = \kappa \ln \kappa.   
    \]

    $\ell = 2$. Follows by taking $\mathbf{x} = \mathbb{E}_{\tau \sim \pi}[\mathbf{G}(\tau)]$ in Part 1.

    \textbf{Part 3.} Note that $w(p) = \ln v^{(\ell)}_*(p) = \ln \max_{\pi \in \Pi} v^{(\ell)}(\pi, p) = \max_{\pi \in \Pi} \ln v^{(\ell)}(\pi, p)$. Given distinct $p, q \le 1$, denote $\pi_p = {\arg\max}_{\pi \in \Pi} \ln v^{(\ell)}(\pi, p)$. Then, by Part 1,
    \[
        \frac{w(q) - w(p)}{q - p} = \frac{\left(\max_{\pi \in \Pi} \ln v^{(\ell)}(\pi, q)\right) - \ln v^{(\ell)}(\pi_p, p)}{q - p} \le \frac{\ln v^{(\ell)}(\pi_p, q) - \ln v^{(\ell)}(\pi_p, p)}{q - p} \le \kappa \ln \kappa. \qedhere
    \]
\end{proof}

We are ready to prove Theorem \ref{thm: portfolio-with-line-search} that gives guarantees for $p$-\textsc{MeanPortfolio} (Algorithm \ref{alg:portfolios-for-rlhf}).

\begin{proof}[Proof of Theorem \ref{thm: portfolio-with-line-search}]
    We prove that (correctness) the set $\Pi'$ of policies output by the algorithm is indeed an $\alpha$-approximate portfolio, (size bound) $|\Pi'| = O\left(\frac{\ln \kappa}{\ln(1/\alpha)}\right)$, and (oracle complexity) the number of oracle calls to Problem (\ref{eq:max_aggre}) is upper bounded by
    \begin{equation}\label{eqn: oracle-complexity}
        O\left(\frac{(\ln \kappa)^2 \ln \ln N}{\ln (1/\alpha) \ln \ln (1/\alpha)}\right).
    \end{equation}
    We note that this size bound can be tightened slightly to $O\left(\frac{(\ln \kappa) (\ln \kappa + \ln \ln N)}{\ln (1/\alpha) \ln \ln (1/\alpha)}\right)$; we present the the above cleaner bound for simplicity.

    \paragraph{Correctness.} Suppose the algorithm returns $\Pi' = \{\pi_0, \pi_1, \ldots, \pi_K\}$ such that $\pi_t = {\arg\max}_{\pi \in \Pi} v^{(\ell)}(\pi, p_t)$ with $- \frac{\ln N}{\ln(1/\alpha)} = p_0 < p_1 < \ldots < p_K = 1$.

    Lemma \ref{lem: approximation-min-with-large-negative-p-agg} shows that $\pi_0$ is an $\alpha$-approximation for all $p \le p_0$. It is therefore sufficient to prove that for all $t \in [0, K - 1]$, policy $\pi_t$ is an $\alpha$-approximation for all $p \in [p_t, p_{t + 1}]$. Since $p_{t + 1} = \textsc{LineSearch}(v^{(\ell)}, p_t, \Pi, \alpha)$, Lemma \ref{lem: line-search-approximation-guarantee} implies this.
    
    
    \paragraph{Size bound.} Note that by definition, $p_{t + 1} = \textsc{LineSearch}(v^{(\ell)}, p_t, \Pi, \alpha)$. Therefore, from Lemma \ref{lem: line-search-guarantee}, we have that $\pi_t$ satisfies $v_*^{(\ell)}(p_{t + 1}) \ge \frac{v^{(\ell)}_*(p_t)}{\sqrt{\alpha}}$ for all $t$ except possibly $t = K - 1$. Therefore,
    \[
        v_*^{(\ell)}(p_K) \ge v_*^{(\ell)}(p_{K - 1}) \ge \left(\frac{1}{\alpha}\right)^{(K - 1)/2} v^{(\ell)}_*(p_0) = \left(\frac{1}{\alpha}\right)^{(K - 1)/2} v^{(\ell)}_*(- \infty).
    \]
    However, $v_*^{(\ell)}(p_K) \le HU$ and $v_*^{(\ell)}(p_0) \ge HL$, so that $\frac{v_*^{(\ell)}(p_K)}{v_*^{(\ell)}(p_0)} \le \frac{U}{L} = \kappa$, and therefore,
    \[
        \frac{K - 1}{2} \le \log_{(1/\alpha)} \kappa = \frac{\ln \kappa}{\ln (1/\alpha)}.
    \]

    \paragraph{Oracle complexity.} To bound the oracle complexity, we will show that each run of \textsc{LineSearch} calls the oracle to solve Problem (\ref{eq:max_aggre}) at most $O\left(\ln\left(\kappa |p_0|\right)\right)$ times, where $p_0 = - \frac{\ln N}{\ln(1/\alpha)}$ is the first iterate in $p$-\textsc{MeanPortfolio} and $\kappa = U/L$ is the condition number of the rewards. Since \textsc{LineSearch} is called at most $O(K) = O\left(\frac{\ln \kappa}{\ln(1/\alpha)}\right)$ times, this implies the bound (\ref{eqn: oracle-complexity}).

    To bound the number of oracle calls in \textsc{LineSearch}, we will use Lemma \ref{lem: slope-bound}.3 that upper bounds the slope of $w(p) := \ln v_*^{(\ell)}(p)$ by $m := \kappa \ln \kappa$. Suppose, for a given $p \ge p_0 = - \frac{\ln N}{\ln(1/\alpha)}$ that \textsc{LineSearch} on input $p, \alpha$ finished in $j$ iterations. Then, we have the following for \textsc{LineSearch}:
    \begin{enumerate}
        \item $p_0 \le p \le a < b \le  1$ at all times, and 
        \item except in the last iteration, we have $v(\pi, a) < \alpha \ v_*(b)$ (otherwise \textsc{LineSearch} terminates by step \ref{step: main-algorithm-while-loop}). 
    \end{enumerate}
    Denote by $a^{(j - 1)}, b^{(j - 1)}$ the value of $a, b$ in iteration $j - 1$. Then, since $b - a$ is halved in each iteration, we must have
        \[
            0 < b^{(j - 1)} - a^{(j - 1)} = (1 - p) 2^{-(j - 1)} \le (1 - p_0) 2^{-(j - 1)} \le \frac{4\ln N}{\ln(1/\alpha)} 2^{-j}. 
        \]
        However, since the algorithm does not terminate in step $j - 1$, as discussed, we must also have that
        \begin{equation}\label{eqn: line-search-lower-vs-upper-bound-guarantee-1}
            v^{(\ell)}(\pi, a^{(j - 1)}) < \alpha \cdot v^{(\ell)}_*(b^{(j - 1)}). 
        \end{equation}
        From Lemma \ref{lem: slope-bound}, we get that
        \begin{align}\label{eqn: line-search-slope-bound}
            \ln \frac{v^{(\ell)}_*(b^{(j - 1)})}{v^{(\ell)}_*(a^{(j - 1)})} \le (\kappa \ln \kappa) (b^{(j - 1)} - a^{(j - 1)}) \le \frac{4 \kappa (\ln \kappa)(\ln N)}{\ln(1/\alpha)} 2^{-j}.
        \end{align}
        However, from Lemma \ref{lem: line-search-invariant}, we get $v^{(\ell)}(\pi, a^{(j - 1)}) \ge \sqrt{\alpha} \ v^{(\ell)}_*(a^{(j - 1)})$. Putting these together with eqns. (\ref{eqn: line-search-lower-vs-upper-bound-guarantee-1}) and (\ref{eqn: line-search-slope-bound}), we get that
        \[
            \frac{1}{2} \ln(1/\alpha) \le \frac{4 \kappa (\ln \kappa)(\ln N)}{\ln(1/\alpha)} 2^{-j}.
        \]
        Therefore,
        \[
            2^j \le \frac{8 \kappa (\ln \kappa)(\ln N)}{(\ln(1/\alpha))^2}.
        \]
        Equivalently, the number of iterations $j$ in \textsc{LineSearch} is bounded by
        \[
            O\left(\ln \left(\frac{\kappa \ln N}{\ln(1/\alpha)}\right)\right) = O\left(\frac{(\ln \kappa) (\ln\ln N)}{\ln\ln (1/\alpha)}\right).
        \]
\end{proof}

\subsection{Proof of Proposition \ref{prop:warmstart}}

We prove the proposition assuming the aggregation function $v^{(1)}(\pi, p)$. The same proof applies for $v^{(2)}(\pi, p)$.

\begin{proof}
\begin{align*}
    &\max_{\pi \in \Pi} \mathbb{E}_{\tau \sim \pi} \left[ f(\textbf{G}(\tau),q) \right]
    - \mathbb{E}_{\tau \sim \pi_p} \left[ f(\textbf{G}(\tau),q) \right] \\
    &\leq \max_{\pi \in \Pi} \mathbb{E}_{\tau \sim \pi} \left[ f(\textbf{G}(\tau),q) \right]
    - \mathbb{E}_{\tau \sim \pi_p} \left[ f(\textbf{G}(\tau),p) \right] \\
    &= \max_{\pi \in \Pi} \mathbb{E}_{\tau \sim \pi} \left[ f(\textbf{G}(\tau),q) \right]
    - \max_{\pi \in \Pi} \mathbb{E}_{\tau \sim \pi} \left[ f(\textbf{G}(\tau),p) \right] \\
    &\leq \max_{\pi \in \Pi} \mathbb{E}_{\tau \sim \pi} \left[ f(\textbf{G}(\tau),q) - f(\textbf{G}(\tau),p) \right] \\
    &\leq \max_{\pi \in \Pi} \mathbb{E}_{\tau \sim \pi} \left[ (q-p) \max_{r \in (p,q)} \frac{\mathrm{d} f(\textbf{G}(\tau),r)}{\mathrm{d}r} \right] \\
    &\leq \max_{\pi \in \Pi} \mathbb{E}_{\tau \sim \pi} \left[ (q-p) \kappa \ln{\kappa} H U \right] \\
    &= (q-p) UH\kappa \ln{\kappa}.
\end{align*}

The first inequality follows from Lemma \ref{lem: social-welfare-monotonicity}, which ensures monotonicity in the \(p\) parameter. The equality comes directly from the definition of \(\pi_p\), as it maximizes \(v^{(1)}(\pi,p)\). The second inequality uses the properties of the maximum function. The next line applies the mean-value theorem to the function \( f(\textbf{G}(\tau),p) \) with respect to the variable \(p\). The subsequent bound follows from the boundedness of the reward function and Lemma \ref{lem: slope-bound}. Finally, the terms involving \(\kappa\) are replaced with \(L\) and \(U\) for conciseness.
\end{proof}

\edit{
\section{Generalized $p-$Means: Illustrative Example}
\label{sec:pmeans}

To clarify the meaning behind 
$p$-means, we constructed an illustrative example in Figure \ref{fig:abc-comparison}. In this example, three vectors $A$, $B$ and  $C$ represent different policies, where each vector's 
$i$th entry is the total reward for the $i$th stakeholder. Here, $A$ is a balanced policy, $B$ is unbalanced with the largest sum, and $C$ lies in between. Figure \ref{fig:abc-comparison} (c)  plots the 
$p$-mean values of these vectors, showing that varying $p$
 leads to different optimal policy.

\begin{figure}[t]
  \centering
  \subfigure[\label{fig1}][Entries of $A,B,C$]{
      \includegraphics[width=.32\columnwidth]{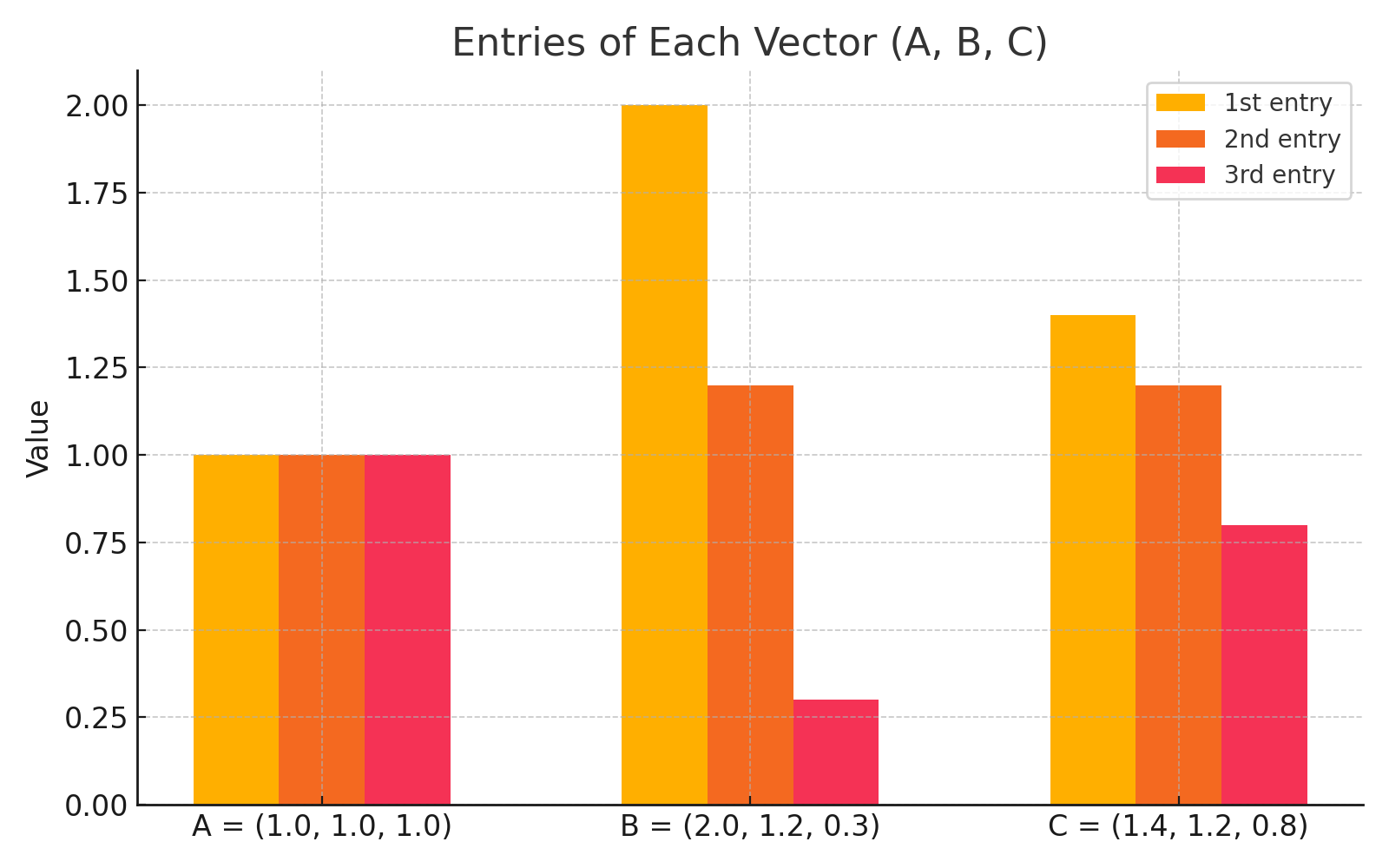}
  }\hfill
  \subfigure[\label{fig2}][$p-$means at selected $p$ values.]{
      \includegraphics[width=.32\columnwidth]{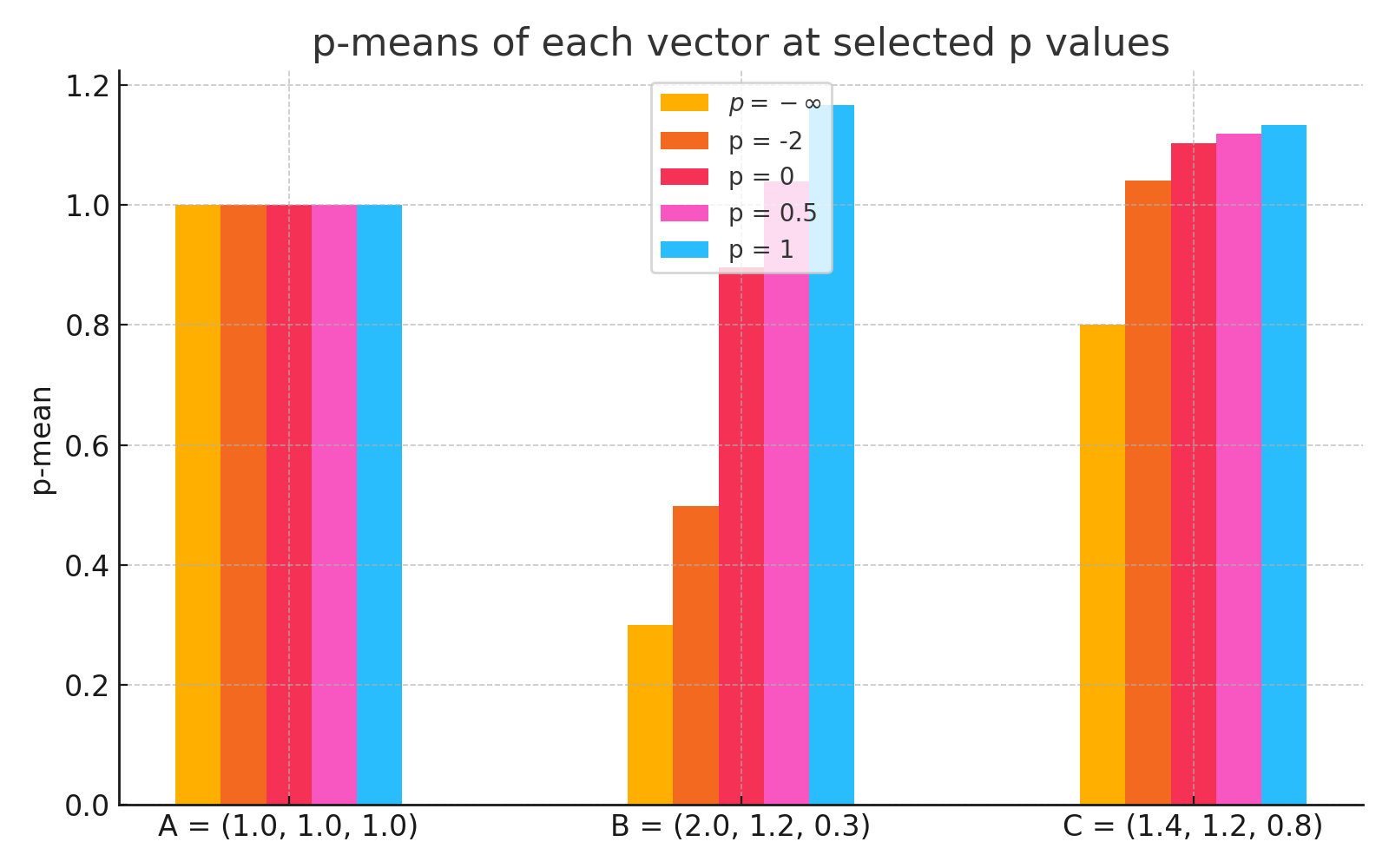}
  }\hfill
  \subfigure[\label{fig3}][$p-$means across different $p$ values.]{
      \includegraphics[width=.32\columnwidth]{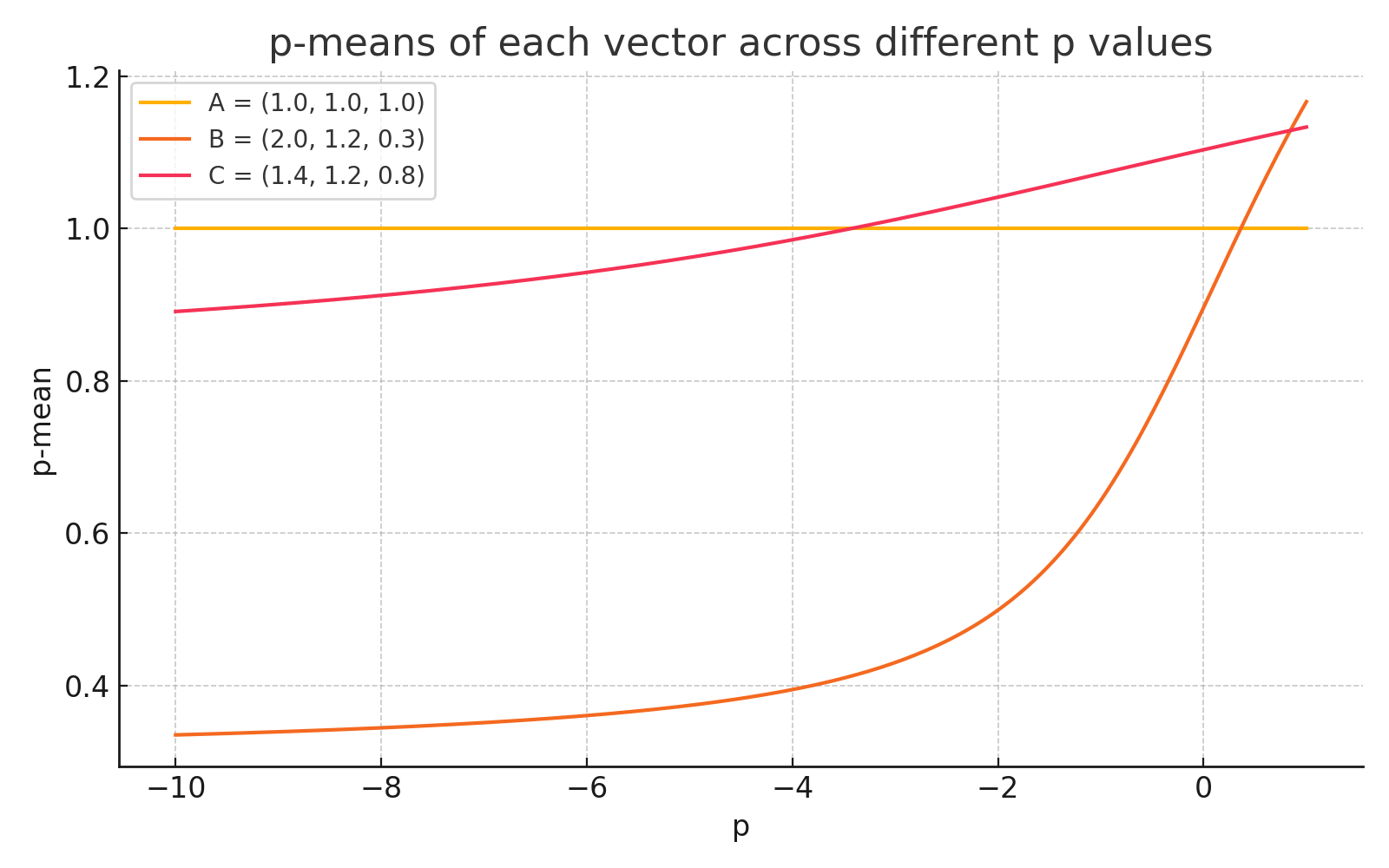}
  }
  \caption{Side-by-side comparison of entries and $p$-means for vectors $A,B,C$.}
  \label{fig:abc-comparison}
\end{figure}

\section{Comparison with other MORL Algorithms}
\label{sec:comparisons}

\cite{Yu24} optimize a generalized Gini welfare assuming a fixed weight vector. We address a fundamentally different scenario, where decision-makers are uncertain about the appropriate fairness criterion (e.g., the choice of 
$p$ in $p$-means or the selection of weights in Gini welfare) in advance. Optimizing a single policy for one fairness parameter can lead to poor outcomes under different criteria (see the approximation qualities of random baselines in our experiments). Our method computes a small, representative set of policies covering the entire spectrum of fairness criteria. This allows decision-makers to select confidently from this set, without worrying about suboptimality under other potential choices of $p$.

\cite{yang19} focuses on weighted linear objectives and learns a single policy, whereas we focus on $p$-means and compute a portfolio (multiple policies). \cite{Alegre+2023}  builds portfolios focusing on weighted linear objectives (as opposed to $p-$means) and does not offer guarantees on portfolio size or oracle complexity. \cite{reymond22} trains a single policy to approximate the Pareto frontier, which does not directly correspond to any social welfare function. To the best of our knowledge, our work is the first to (1) propose a multi-policy approach in 
$p$-means and (2) provide theoretical guarantees on portfolio size, approximation quality, and oracle complexity simultaneously.

In Table \ref{tab: approximations-gpi}, we also present comparisons between $p$-\textsc{MeanPortfolio} and the Generalized Policy Improvement (GPI) algorithm of \cite{Alegre+2023}, which maintains a set of policies $\Pi' \subseteq \Pi$ and iteratively chooses the weight vector $\mathbf{w}$ with the highest difference between the optimal policy value for $\mathbf{w}$ in $\Pi$ vs in $\Pi'$. Our implementation of GPI uses the differential evolution solver from scipy for this global optimization step. The results in the table indicate that our approach achieves significantly better approximation quality. 
We omit other details of GPI here and refer the interested reader to \cite{Alegre+2023}.  

\begin{table}[t]
\caption{A comparison of actual approximation ratios across various portfolio sizes for $p$-\textsc{MeanPortfolio} and our implementation of GPI.}
\label{tab: approximations-gpi}
\centering
\begin{tabular}{|c|c|c|}
\hline
\begin{tabular}[c]{@{}c@{}}Portfolio\\ Size\end{tabular} &
  \begin{tabular}[c]{@{}c@{}}$p$-\textsc{Mean}-\\ \textsc{Portfolio}\end{tabular} &
  GPI \\ \hline
\multicolumn{3}{c}{\rule{0pt}{10pt} Resource Allocation after Natural Disaster} \\ \hline
1 & \textbf{0.706} & 0.514 \\ \hline
2 & \textbf{0.904} & 0.520 \\ \hline
3 & \textbf{0.999} & 0.520 \\ \hline
4 & \textbf{0.100} & 0.520 \\ \hline
 \multicolumn{3}{c}{\rule{0pt}{10pt} Healthcare Intervention} \\ \hline
1 & \textbf{0.924} & 0.347 \\ \hline
2 & \textbf{0.982} & 0.641 \\ \hline
3 & \textbf{0.982} & 0.641 \\ \hline
4 & \textbf{0.982} & 0.641 \\ \hline
5 & \textbf{0.993} & 0.641 \\ \hline
6 & \textbf{0.999} & 0.641 \\ \hline
7 & \textbf{1.000} & 0.641 \\ \hline
\end{tabular}
\end{table}

\section{Additional Experimental Results}
\label{sec:additional_results}

Figures \ref{fig: pvalues_taxi} and \ref{fig: fraction-of-need-met-natural-disaster} illustrate how the optimal policies for different $p$ values in the portfolio lead to varying impacts for the stakeholders. These portfolios are all obtained using $p$-\textsc{MeanPortfolio}. We also include the values of $p$ chosen by $p$-\textsc{MeanPortfolio} to construct the portfolio in Table \ref{tab: p-values}.

\begin{table}[]
\caption{The set of values of $p$ chosen by $p$-\textsc{MeanPortfolio} to obtain the corresponding portfolios.}
\centering
\label{tab: p-values}
\begin{tabular}{|cc|}
\hline
\multicolumn{1}{|c|}{\begin{tabular}[c]{@{}c@{}}Portfolio\\ size\end{tabular}} & $p$ Values                                            \\ \hline

\multicolumn{2}{|c|}{Resource Allocation After Natural Disaster} \\ \hline
\multicolumn{1}{|c|}{1} & -0.829                                \\ \hline
\multicolumn{1}{|c|}{2} & -4.864, -0.466                        \\ \hline
\multicolumn{1}{|c|}{3} & -11.136, -2.034, 0.621                \\ \hline
\multicolumn{1}{|c|}{4} & -15.29, -2.054, 0.517, 0.758          \\ \hline
\multicolumn{2}{|c|}{Healthcare Intervention}                   \\ \hline
\multicolumn{1}{|c|}{1} & -1.325                                \\ \hline
\multicolumn{1}{|c|}{2} & -2.864, 0.034                         \\ \hline
\multicolumn{1}{|c|}{3} & -4.333, -0.333, 0.333                 \\ \hline
\multicolumn{1}{|c|}{4} & -6.641, -0.433, -0.075, 0.463         \\ \hline
\multicolumn{1}{|c|}{5} & -9.216, -4.108, -0.437, -0.078, 0.461 \\ \hline
\multicolumn{1}{|c|}{6}                                                        & -13.801, -6.4, -0.594, -0.22, 0.085, 0.542            \\ \hline
\multicolumn{1}{|c|}{7}                                                        & -17.793, -3.698, -0.549, -0.186, -0.038, 0.222, 0.611 \\ \hline
\multicolumn{2}{|c|}{Taxi Environment} \\ \hline
\multicolumn{1}{|c|}{1} & -2.0                                \\ \hline
\multicolumn{1}{|c|}{2} & -3.89, 0.87                       \\ \hline
\multicolumn{1}{|c|}{3} & -6.21, 0.72, 0.86               \\ \hline
\multicolumn{1}{|c|}{8} & -13.16, -10.06, -6.17, -4.82, 0.66, 0.7, 0.77, 0.89       \\ \hline
\multicolumn{1}{|c|}{10} & -27.03, -13.02, -6.25, 0.66, 0.71, 0.78, 0.81, 0.86, 0.89, 0.95   \\ \hline
\end{tabular}
\end{table}

\begin{figure}
    \centering
    \includegraphics[width=0.6\linewidth]{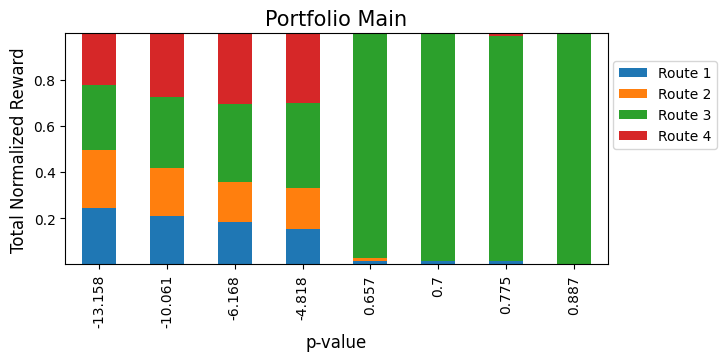}
    \caption{Normalized total reward for each route under different policies in the portfolio generated by $p$-\textsc{MeanPortfolio} in the taxi environment.}
    \label{fig: pvalues_taxi}
\end{figure}
}

\begin{figure}
    \centering
    \includegraphics[width=0.6\linewidth]{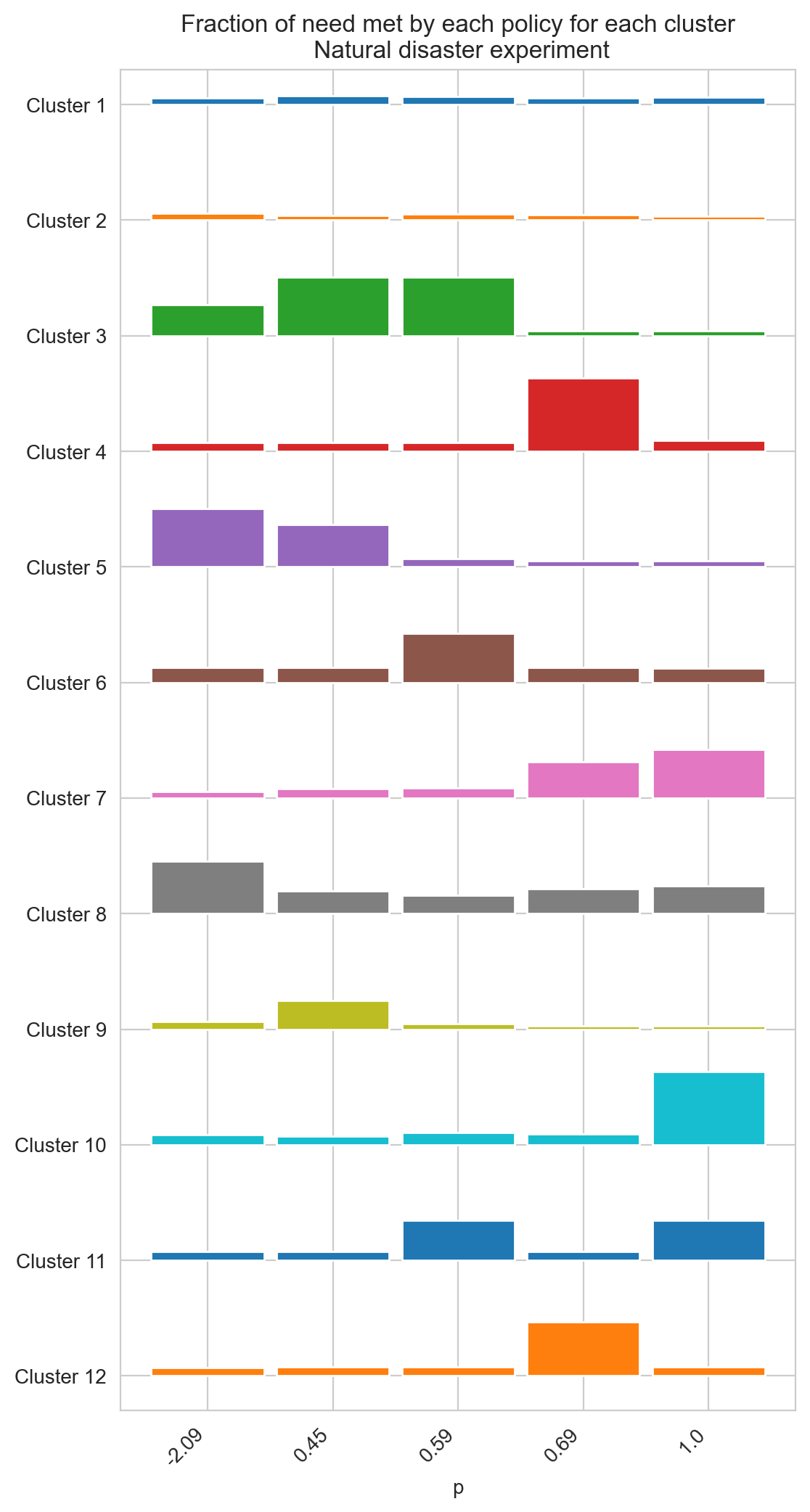}
    \caption{Fraction of need meet for various clusters by various policies in the portfolio obtained by $p$-\textsc{MeanPortfolio} after $H = 4$ intervention steps for the natural disaster experiment. Note that different solutions in the portfolio emphasize different clusters.}
    \label{fig: fraction-of-need-met-natural-disaster}
\end{figure}

\section{Experimental Details}
\label{sec:expdetail}

We include the details of various experiments here.

\subsection{Taxi Environment}
\subsection*{Problem Setting}
This environment consists of a taxi agent whose task is to deliver passengers from source to destination. The world consists of a 6x6 grid and based on the environment setup of \cite{fan2022welfare}, we consider 4 source-destination pairs. Whenever a taxi moves to a source point, it can pick a passenger, move to destination and drop the passenger, thus receiving a reward. Since some routes could be easier to serve, the agent has to decide which route to serve more often. Rewards are multi-dimensional, each dimension corresponding to each route. 
We consider the following source-destination pairs:
$source\_coords = [[0,0], [0,5], [3,0], [1,0]]$,
$destination\_coords = [[1,5], [5,0], [3,3], [0,3]]$
\subsection*{MDP Structure}
\subsubsection*{State Space (\(S\))}
The state space consists of information on the location of the taxi, location of passengers and whether passengers have been picked up by the taxi at the moment.
\subsubsection*{Action Space (\(S\))}
The agent can take one of 6 actions: move north, south, east, west and pick or drop a passenger in the current grid cell.
\subsubsection*{Reward (\(R\))}
The reward function gives reward 0 for moving, -10 reward for taking invalid action of picking or dropping a passenger at wrong coordinating, and 30 reward for dropping a passenger at the right destination.
\subsubsection*{Learning Policy}
Given a scalarization function (or welfare function ), the task is to find a policy $\pi^*$ that maximizined the expected value of scalarized return. Specifically:
\begin{align}
    \pi^* &= \arg\max_{\pi} \mathbb{E}_{\tau \sim \pi}\left[f\big(\textbf{G}(\tau), p\big)\right].
\end{align}

We learn this policy using the Welfare Q-Learning algorithm proposed in \cite{fan2022welfare}. For every problem setting, we train the policy for 200 episodes. Every episode is a finite horizon problem with 1000 timesteps. We consider discount factor gamma as 0.99.

\subsection{Natural Disaster}\label{subsec:nat-disaster}

\subsection*{Problem Setting}
In this synthetically generated example, suppose that in the wake of a natural disaster, a centralized aid agency must determine how to allocate resources to a set of 12 clusters ($N = 12$). Each cluster is characterized by their population density (high or low), proximity to critical infrastructure (near or far), and the predominant income level of its residents (low, middle, or high). 
%
\begin{table*}[ht]
\centering\small
\begin{tabular}{@{}|c|c|c|c|c|c|@{}}
\hline
\textbf{Cluster ID} & \textbf{Density} & \textbf{Proximity} & \textbf{Income Level} & \textbf{Total Population} & \textbf{Initial Need} \\
\hline
1  & High & Far  & High-Income   & 148  & 150  \\
2  & High & Far  & Low-Income    & 307  & 500  \\
3  & High & Far  & Middle-Income & 616  & 650  \\
4  & High & Near & High-Income   & 816  & 300  \\
5  & High & Near & Low-Income    & 1405 & 1000 \\
6  & High & Near & Middle-Income & 2782 & 950  \\
7  & Low  & Far  & High-Income   & 74   & 1000 \\
8  & Low  & Far  & Low-Income    & 203  & 350  \\
9  & Low  & Far  & Middle-Income & 396  & 300  \\
10 & Low  & Near & High-Income   & 36   & 50   \\
11 & Low  & Near & Low-Income    & 113  & 100  \\
12 & Low  & Near & Middle-Income & 230  & 100  \\
\hline
\end{tabular}
\caption{Clustered population data including density, proximity, income level, total population, and initial need.}\label{tab:synth-data}
\end{table*}

\subsection*{MDP Structure}

The problem is formulated as an MDP $\mathcal{M}$, with the following components:

\subsubsection*{State Space (\(S\))}
The state \(s = (s_1, s_2, \dots, s_N)\) represents the current resource needs of the \(N\) clusters. Each \(s_c\) denotes the total remaining need for cluster \(c\). The state space is discretized, with each \(s_c\) taking values in increments of $b$ in $[0, B]$. In our setting, $b = 50$ and $B = 150$.

\subsubsection*{Action Space (\(A\))}
Actions \(a = (a_1, a_2, \dots, a_N)\) represent the allocation of resources to clusters, where: \(a_c \in \{0, b, 2b, \dots, B\}\) represents the resources allocated to cluster \(c\) and the total allocation \(\sum_{c=1}^N a_c\) cannot exceed the total budget $B$.

\subsubsection*{Transition Probability Matrix (TPM, \(P(s'|s, a)\))}
The state transitions depend on the current state \(s\) and the action \(a\). If \(a_c \geq s_c\), \(s_c' = 0\) with probability 1 (this cluster's need is fully met). If \(a_c < s_c\) the unmet need will reduce by the allocation, that is \(s_c' = s_c - a_c\), with 70\% probability. However, there is a 30\% probability that the cluster will observe an increase in need by $b$ units, meant to reflect ballooning needs when immediate action is not taken.

\subsubsection*{Generation of Policies \& Reward Functions}
We consider combinations of ``reasonable" policies a decision-maker might take in this setting. First, a valid policy would not allocate resources to a zero-need cluster. Secondly, we consider the following priorities that a decision-maker might have when allocating funds among nonzero unmet need clusters: (1) lowest income, (2) highest population, (3) highest unmet need, (4) highest unmet need per capita, (5) high population density regions, and (6) far from critical infrastructure. In every time period, for every feasible state, we apply one of these policies, a random convex weighting of these policies, or a randomized priority to generate 10,000 different feasible policies.

The expected reward accrued by each cluster is the average of the sum of the expected fraction of the initial need met under a policy and the fraction of the expected total allocation awarded to that cluster over the horizon.




\subsection{Healthcare Intervention}
The Healthcare Intervention problem is based on the large-scale mobile-health program run by the NGO ARMMAN \cite{ARMMAN}. The goal of the program is to maximize engagement of beneficiaries with the program using limited service call interventions. \edit{Based on previous works, we model this problem as RMAB problem where we have multiple arms or beneficiaries, a budget on the number of service calls to give every week.} For every beneficiary, we have information on their listenership with the voice call. This is considered as an engaging or non-engaging state if the listenership is above or below 30-seconds threshold respectively. One week of time is considered as one timestep. Finally, every week, the action can be to place (active) or not to place a live service call (passive). Additionally, for every beneficiary, we have information on their socio-demographic characteristics such as age, income, education. We use real world data from service quality improvement conducted by ARMMAN in January 2022 \cite{verma2023expanding}.
Further, for our experiments, we sample data for $2100$ beneficiaries, and run the experiment for $H=12$ timesteps.

\subsection*{Generation of Policies \& Reward Functions}
The default reward incentivizes agents to be in engaging state. Thus, agents receive reward 1 if they are in engaging state, and 0 otherwise. However, due to varying policy needs over time or over different geographies, health workers often have to prioritize some beneficiaries more than others. 
\edit{To capture these varying priorities, we leverage the Social Choice Language Model framework \cite{verma2024} to derive reward functions from natural-language commands. Given a command specifying $N $preferences (each indicating a subgroup to prioritize) this method generates two sets of reward functions: 
\begin{enumerate}
    \item \textbf{Individual preferences:} one reward function per preference (total of $N$ reward functions) which we treat as each stakeholder's reward function.
    \item \textbf{Balancing functions:} a collection of reward functions that trade off among the $N$ preferences, whose corresponding policies form our feasible set $\Pi$.
\end{enumerate}
In our experiments, we set $N = 59$ and construct  $|\Pi| = 200$ balancing policies as feasible set.}

\subsection*{Learning Policy}
It is computationally intractable to optimally solve the RMAB problem \cite{papadimitriou1999complexity}. Thus, we use the popular Whittle Index Heuristic \cite{whittle98} to solve RMAB problem for a given single reward function. Specifically, Whittle Index quantifies the reward gain achieved for every arm in every state if we took the active action as compared to if we took the passive action. The policy then chooses the arms with highest whittle indices within the budget limit. 

\subsection*{Baseline Policy}
The baseline policy mentioned in Figure \ref{fig: sclm-education-age} is trained on the default reward.

\subsection{Other Details}
\edit{
\subsection*{Initial $p$ for \textsc{Budget-ConstrainedPortfolio}}
For \textsc{Budget-ConstrainedPortfolio}, we fix $p_0 = - 100$.

\subsection*{Computing actual approximation factor}
The approximation factor $\mathcal{Q}(\Pi')$ for any portfolio $\Pi'$ is computed via a grid search over 1000 points in $(\infty, 1]$ for the natural disaster environment and the healthcare intervention problem. For the taxi environment, we used 50 points due to the computational burden of solving each problem.

\subsection*{Choice of $\alpha$}
For a given portfolio size $K$, we chose the smallest $\alpha \in \{0.05, 0.10, 0.15, \dots,0.90, 0.95, 0.99\}$ among which the resulting portfolio has size $K$. }



\end{document}